\documentclass{article}

\usepackage{times}
\usepackage{graphicx} 
\usepackage{subfigure} 

\usepackage{natbib}

\usepackage{algorithm}

\usepackage{hyperref}

\usepackage[noend]{algpseudocode}
\usepackage{enumitem}
\usepackage{array,multirow}
\usepackage{booktabs}
\usepackage{soul}
\usepackage[usenames,dvipsnames]{xcolor}
\usepackage{relsize}
\sethlcolor{SkyBlue}
\usepackage{xr}
\usepackage[titletoc,title]{appendix}

\usepackage{amsmath}
\usepackage{amsthm}
\usepackage{amsfonts}
\usepackage{amssymb}

\DeclareMathOperator*{\argmin}{arg\,min}
\DeclareMathOperator*{\argmax}{arg\,max}

\newtheoremstyle{normal}
  {.2\baselineskip\@plus.05\baselineskip\@minus.05\baselineskip}
  {.05\baselineskip\@plus.05\baselineskip\@minus.05\baselineskip}
  {\itshape}
  {}
  {\bfseries}
  {.}
  { }
  {}

\newtheoremstyle{cited}%
  {.2\baselineskip\@plus.05\baselineskip\@minus.05\baselineskip}
  {.05\baselineskip\@plus.05\baselineskip\@minus.05\baselineskip}
  {\itshape}
  {}
  {\bfseries}
  {.}
  {.5em}
  {\thmname{#1} \thmnumber{#2} \thmnote{\normalfont#3}}

\theoremstyle{normal}

\newtheorem{thm}{Theorem}

\newtheorem{corollary}{Corollary}
\newtheorem{define}{Definition}

\theoremstyle{cited}

\newtheorem{citedcorollary}{Corollary}

\newcommand{\tuple}[1]{\ensuremath{\left \langle #1 \right \rangle }}

\newcommand{\alg}[1]{\textnormal{\textsc{#1}}}

\newcommand{\R}{\mathbb{R}}
\newcommand{\N}{\mathbb{N}}
\newcommand{\B}{\mathbb{B}}

\newcommand{\mb}[1]{\mathbf{#1}}
\newcommand{\mc}[1]{\mathcal{#1}}

\newcommand{\mtt}[1]{\mathtt{#1}}

\newcommand{\x}{\mathbf{x}}
\newcommand{\y}{\mathbf{y}}

\newcommand{\vb}{\mathbf{v}}

\newcommand{\Xb}{\mathbf{X}}
\newcommand{\Yb}{\mathbf{Y}}
\newcommand{\Zb}{\mathbf{Z}}
\newcommand{\Cb}{\mathbf{C}}

\newcommand{\X}{\mathcal{X}}
\newcommand{\Y}{\mathcal{Y}}
\newcommand{\Z}{\mathcal{Z}}

\usepackage{etoolbox}
\makeatletter
\newcount\c@additionalboxlevel
\newcount\c@maxboxlevel
\patchcmd\@combinedblfloats{\box\@outputbox}{%
  \stepcounter{additionalboxlevel}%
  \box\@outputbox
}{}{\errmessage{\noexpand\@combinedblfloats could not be patched}}

\AtBeginShipout{%
  \ifnum\value{additionalboxlevel}>\value{maxboxlevel}%
    \typeout{Warning: maxboxlevel might be too small, increase to %
      \the\value{additionalboxlevel}%
    }%
  \fi 
  \@whilenum\value{additionalboxlevel}<\value{maxboxlevel}\do{%
    \typeout{* Additional boxing of page `\thepage'}%
    \setbox\AtBeginShipoutBox=\hbox{\copy\AtBeginShipoutBox}%
    \stepcounter{additionalboxlevel}%
  }%
  \setcounter{additionalboxlevel}{0}%
}
\makeatother

\newcommand{\overbar}[1]{\mkern 1.5mu\overline{\mkern-1.5mu#1\mkern-1.5mu}\mkern 1.5mu}
\newcommand{\ol}[1]{\overbar{#1}}

\newcommand{\sr}{R}
\newcommand{\compat}{\sim}
\newcommand{\sumspf}{SumSPF}

\newcommand{\JI}{\mb{1}_R}
\newcommand{\bigbowtie}{\mathlarger{\mathlarger{\bowtie}}}

\newcommand{\shSAT}{\#\text{SAT}}

\mathchardef\mhyphen="2D
\DeclareMathOperator*{\argoplus}{arg\,\hspace{-0.1em}\oplus}

\DeclareMathOperator*{\argOR}{arg\,\bigvee}

\usepackage[accepted]{icml2016}

\icmltitlerunning{The Sum-Product Theorem: A Foundation for Learning Tractable Models}

\begin{document} 

\twocolumn[
\icmltitle{The Sum-Product Theorem: A Foundation for Learning Tractable Models}

\vspace{-1em}
\icmlauthor{Abram L. Friesen}{afriesen@cs.washington.edu}
\icmlauthor{Pedro Domingos}{pedrod@cs.washington.edu}
\icmladdress{Department of Computer Science and Engineering, 
            University of Washington,
            Seattle, WA 98195 USA}

\icmlkeywords{tractable, sum-product network, inference, learning, semiring, treewidth}

\vskip 0.15in
]

\begin{abstract} 
Inference in expressive probabilistic models is generally intractable, 
which makes them difficult to learn and limits their applicability.
Sum-product networks
are a class of deep
models where, surprisingly, inference remains tractable even when an
arbitrary number of hidden layers are present. In this paper, we
generalize this result to a much broader set of learning problems: all
those where inference consists of summing a function over a semiring.
This includes satisfiability, constraint satisfaction, optimization,
integration, and others. In any semiring, for summation to be
tractable it suffices that the factors of every product have disjoint
scopes. This unifies and extends many previous results in the
literature. Enforcing this condition at learning time thus ensures that
the learned models are tractable. We illustrate the power and
generality of this approach by applying it to a new type of structured
prediction problem: learning a nonconvex function that can be globally
optimized in polynomial time. We show
empirically that this greatly outperforms the standard approach of
learning 
without regard to the cost of optimization.
\end{abstract} 

\section{Introduction}
\label{sec:intro}

Graphical models are a compact representation often used
as a target for learning probabilistic models. Unfortunately, inference in them
is exponential in their treewidth~\citep{Chandrasekaran2008}, a common
measure of complexity. Further, since inference is a subroutine of learning, graphical
models are hard to learn unless restricted to those with 
low treewidth~\citep{Bach2001, Chechetka2008}, but few real-world
problems exhibit this property.
Recent research, however, has shown that probabilistic models can in fact 
be much more expressive than this while remaining 
tractable~\citep{ltpm2014}. 
In particular, sum-product networks (SPNs)~\citep{Gens2013,Poon2011} are a class of
deep probabilistic models that consist of many layers of hidden variables and can have
unbounded treewidth.
Despite this, inference in SPNs is guaranteed to be tractable, 
and their structure and parameters can be effectively and accurately learned 
from data~\citep{Gens2012,Gens2013,Rooshenas2014}.

In this paper, we generalize and extend the ideas behind SPNs to enable 
learning tractable high-treewidth representations for a
much wider class of problems,
including satisfiability, MAX-SAT, model counting,
constraint satisfaction, marginal and MPE inference, integration,
nonconvex optimization, database querying, and first-order probabilistic inference.
The class of problems we address can be viewed as generalizing
structured prediction beyond combinatorial optimization~\citep{Taskar2005}, 
to include optimization for continuous models and others.
Instead of approaching each domain individually, we build on a long line of
work showing how, despite apparent differences, these problems
in fact have much common structure
(e.g.,~\citet{Bistarelli1997,Dechter1999,Aji2000,Wilson2005,Green2007,Dechter2007,Bacchus2009}); 
namely, that each consists of summing a function over a semiring.
For example, in the Boolean semiring the sum and product operations are disjunction
and conjunction, and deciding satisfiability is summing a Boolean formula
over all truth assignments. MPE inference is summation over all states in the
max-product semiring, etc.

We begin by identifying and proving the sum-product theorem, 
a unifying principle for tractable inference that states a simple sufficient
condition for summation to be tractable in any semiring: that the factors of every
product have disjoint scopes. In ``flat'' representations like graphical models and
conjunctive normal form, consisting of a single product of sums, 
this would allow only trivial models;
but in deep representations like SPNs and negation normal form
it provides remarkable flexibility.
Based on the sum-product theorem, we develop an algorithm for learning 
representations that satisfy this condition, thus guaranteeing that the learned 
functions are tractable yet expressive.
We demonstrate the power and generality of our approach by 
applying it to a new type of structured prediction problem: learning a nonconvex
function that can be optimized in polynomial time.
Empirically, we show that this greatly outperforms the standard approach of
learning a continuous function without regard to the cost of optimizing it. 
We also show that a number of existing and novel results are corollaries
of the sum-product theorem, 
propose a general algorithm for inference in any semiring,
define novel tractable classes of constraint satisfaction problems, 
integrable and optimizable functions, and database queries, and
present a much simpler proof of the tractability of tractable Markov logic.

\setlength{\textfloatsep}{2pt plus 0.0pt minus 2.0pt}
\setlength{\floatsep}{2.0pt plus 2.0pt minus 2.0pt}
\setlength{\intextsep}{4.0pt plus 2.0pt minus 1.0pt}

\vspace{-0.6em}
\section{The sum-product theorem}
\label{sec:spt}

We begin by introducing our notation and defining several important concepts.
We denote a vector of variables by $\Xb = (X_1, \dots, X_n)$ and its value by 
$\x = (x_1, \dots, x_n)$ for $x_i \in \X_i$ for all $i$,
where $\X_i$ is the domain of $X_i$. We denote subsets 
(for simplicity, we treat tuples as sets)
of variables as $\Xb_A, \Xb_a \subseteq \Xb$, 
where the domains $\X_A, \X_a$ are the Cartesian product
of the domains of the variables in $\Xb_A, \Xb_a$, respectively.
We denote (partial) assignments as $\mb{a} \in \X_A$ and restrictions of these to 
$\Xb_B \subset \Xb_A$ as $\mb{a}_B$. 
To indicate compatibility between $\mb{a} \in \X_A$ and $\mb{c} \in \X_C$ (i.e., that
$\mb{a}_j = \mb{c}_j$ for all $X_j \in \Xb_A \cap \Xb_C$), we write $\mb{a} \compat \mb{c}$. 
The \emph{scope} of a function is the set of variables
it takes as input.

\begin{define}
A \emph{commutative semiring} $(\sr, \oplus, \otimes, 0, 1)$ 
is a nonempty set $\sr$ on which the operations of sum ($\oplus$) 
and product ($\otimes$) are defined and satisfy the following conditions:
(i) $(\sr,\oplus)$ and $(\sr,\otimes)$ are associative and commutative, 
with identity elements $0,1 \in R$ such that $0 \neq 1$, $a \oplus 0 = a$, and $a \otimes 1 = a$ for all $a \in \sr$;
(ii) $\otimes$ distributes over $\oplus$, such that
$a \otimes ( b \oplus c ) = ( a \otimes b ) \oplus ( a \otimes c )$
for all $a, b, c \in \sr$; and
(iii) $0$ is absorbing for $\otimes$, such that $a \otimes 0 = 0$ for all $a \in \sr$.
\end{define}
%

\vspace{-0.3em}
We are interested in computing summations $\bigoplus_{\x \in \X} F(\x)$, for
$(\sr, \oplus, \otimes,0,1)$ a commutative semiring and $F : \X \rightarrow \sr$ a function 
on that semiring, with
$\X$ a finite set (but see Section~\ref{sec:continuous} for extensions to continuous variables).
We refer to such a function as a sum-product function.
\begin{define}
\hspace{-0.3em}
A \emph{sum-product function} (SPF) 
over $(\sr,\hspace{-0.1em}\Xb,\hspace{-0.1em}\Phi)$, 
where $\sr$ is a semiring, $\Xb$ is a set of variables, 
and $\Phi$ is a set of constant ($\phi_l \in \sr$) and univariate functions 
($\phi_l : \X_{j} \rightarrow \sr$ for $X_j \in \Xb$), 
is any of the following:
(i) a function $\phi_l \in \Phi$,
(ii) a product of SPFs, or
(iii) a sum of SPFs.
\end{define}
An SPF $S(\Xb)$ computes a mapping $S : \X \rightarrow \sr$ 
and can be represented 
by a rooted directed acyclic graph (DAG), where each
leaf node is labeled with a function $\phi_l \in \Phi$ and
each non-leaf node is labeled with either $\oplus$
or $\otimes$ and referred to as a sum or product node, respectively. 
Two SPFs are \emph{compatible} iff they compute the
same mapping; i.e., $S_1(\x) = S_2(\x)$ for all $\x \in \X$, where
$S_1(\Xb)$ and $S_2(\Xb)$ are SPFs.
The \emph{size} of an SPF is the number of edges in the graph. 
The DAG rooted at each node $v \in S$ represents a sub-SPF
${S_v : \X_v \rightarrow \sr}$ for $\Xb_v \subseteq \Xb$.
Notice that restricting the leaf functions $\phi_l$ to be univariate incurs 
no loss of generality because any mapping $\psi : \X \rightarrow \sr$
is compatible with the trivial SPF 
$F(\Xb) = \bigoplus_{\x \in \X} ( \psi(\x) \otimes \bigotimes_{i = 1}^n [X_i = \x_i] )$,
where the indicator function $[.]$ has value $1$ when its argument is true, and $0$ 
otherwise (recall that $0$ and $1$ are the semiring identity elements).
SPFs are similar to arithmetic circuits~\citep{Shpilka2010}, but the leaves
of an SPF are functions instead of variables. 
\citet{Darwiche2003} used arithmetic circuits as a data structure to support
inference in Bayesian networks over discrete variables.
An important subclass of SPFs are those that are decomposable.\looseness=-1

\begin{define}
A product node is \emph{decomposable} iff the scopes of its children are disjoint.
An SPF is decomposable iff all of its product nodes are decomposable.
\end{define}

\noindent
Decomposability 
is a simple condition
that defines a class of functions for which inference is tractable.

\begin{thm}[Sum-product theorem]
Every decomposable SPF can be summed in time linear in its size.
\label{thm:spt}
\end{thm}
\vspace{-\baselineskip}

\begin{proof}
The proof is recursive, starting from the leaves of the SPF.
Let $S(\Xb)$ be a decomposable SPF on commutative semiring $(\sr, \oplus, \otimes,0,1)$.
Every leaf node can be summed in constant time, 
because each is labeled with either a constant or univariate function.
Now, let $v \in S$ be a node, with $S_v(\Xb_v)$ the sub-SPF rooted at $v$ and
$Z_v = \bigoplus_{\X_{v}} S_{v}(\Xb_v)$ its summation.
Let $\{ c_i \}$ be the children of $v$ for $c_i \in S$, with sub-SPFs 
$S_{i}(\Xb_{i})$ for $\Xb_{i} \subseteq \Xb_v$ and summations $Z_{i}$.
Let $\X_{v \backslash i}$ be the domain of variables $\Xb_v \backslash \Xb_i$.
If $v$ is a sum node, then 
$Z_v = \bigoplus_{\X_v} \bigoplus_{i} S_{i}(\Xb_{i}) 
      = \bigoplus_{i} \bigoplus_{\X_v} S_{i}(\Xb_{i})
      = \bigoplus_{i} \bigoplus_{\X_{v \backslash i}} \bigoplus_{\X_{i}} S_{i}(\Xb_{i})
      = \bigoplus_{i} Z_{i} \otimes \big( \bigoplus_{\X_{v \backslash i}} 1 \big).$
If $v$ is a product node, then any two children $c_i, c_j$ for $i,j \in \{1,\dots,m\}$ have disjoint scopes,
$\Xb_{i} \cap \Xb_{j} = \varnothing$, and 
$Z_v = \bigoplus_{\X_v} \bigotimes_{i} S_{i}(\Xb_{i}) 
      = \bigoplus_{\X_1} \bigoplus_{\X_{v \backslash 1}} \bigotimes_i S_i(\Xb_i)
      = \bigoplus_{\X_1} S_1(\Xb_1) \otimes \bigoplus_{\X_{v \backslash 1}} \bigotimes_{i=2}^m S_i(\Xb_i)
      = \bigotimes_i \bigoplus_{\X_{i}} S_{i}(\Xb_{i})
      = \bigotimes_i Z_{i}.$
The above equations only require associativity and commutativity of $\oplus$ and 
associativity and distributivity of $\otimes$, which are properties of a semiring.
Thus, any node can be summed over its domain in time linear in the number of its children,
and $S$ can be summed in time linear in its size. \looseness=-1
\end{proof}
\vspace{-0.8em}

%
%
%

\let\OldUrlFont\UrlFont \renewcommand{\UrlFont}{\fontsize{8.5pt}{2pt}\selectfont}

We assume here that
$\bigoplus_{\X_{v \backslash i}} 1$ can be
computed in constant time and that each leaf function can be evaluated in constant time, 
which is true for all semirings considered. 
We also assume that $a \oplus b$ and $a \otimes b$ take constant time for any elements 
$a, b$ of semiring $R$, which is true for most common semirings.
See Appendix~\ref{sec:complexity-sup} for details.
%
%

\renewcommand{\UrlFont}{\OldUrlFont}

The complexity of summation in an SPF can be related to other notions of complexity,
such as treewidth, the most common and relevant complexity measure across the domains we consider.
To define the treewidth of an SPF, we first define junction trees~\citep{Lauritzen1988,Aji2000}
and a related class of SPFs.
\begin{define}
A \emph{junction tree} over variables $\Xb$
is a tuple $(T,Q)$, where $T$ is a rooted tree, $Q$ is a set of subsets of variables,
each vertex $i \in T$ contains a subset of variables
$\Cb_i \in Q$ such that $\cup_i \Cb_i = \Xb$,
and for
every pair of vertices $i,j \in T$ and for all $k \in T$ on the (unique) path
from $i$ to $j$, 
$\Cb_i \cap \Cb_j \subseteq \Cb_k$.
The separator for an edge $(i,j) \in T$
is defined as $\mb{S}_{ij} = \Cb_i \cap \Cb_j$. 
\end{define}
\vspace{-0.1em}

A junction tree provides a schematic for constructing
a specific type of decomposable SPF called a tree-like SPF
(a semiring-generalized version of a construction from~\citet{Darwiche2003}). 
Note that a tree-like SPF is not a tree, however, as many of its nodes have multiple parents.
\begin{define}
A \emph{tree-like SPF} over variables $\Xb$ is constructed from a
junction tree $\mc{T} = (T,Q)$ and functions 
$\{ \psi_i(\Cb_i) \}$ where $\Cb_i \in Q$ and $i \in T$,
and contains the following nodes:
(i) a node $\phi_{vt}$ with indicator 
$\phi_{t}(X_v) = {[X_v = t]}$
for each value $t \in \X_v$ of each variable $X_v \in \Xb$; 
(ii) a (leaf) node $a_i$ with value $\psi_i(\mb{c}_i)$ and a product node $c_i$
for each value $\mb{c}_i \in \X_{\Cb_i}$ of each cluster $\Cb_i$; 
(iii) a sum node $s_{ij}$ for each value $\mb{s}_{ij} \in \X_{\mb{S}_{ij}}$ 
of each separator $\mb{S}_{ij}$, and (iv) a single root sum node $s$.
\vspace{-0.15em}
A product node $c_j$ and a sum node $s_{ij}$ are compatible 
iff their corresponding values are compatible; i.e., $\mb{c}_j \compat \mb{s}_{ij}$.
The nodes are connected as follows.
The children of the root $s$ are all product nodes $c_r$ for $r$ the root of $T$.
The children of product node $c_j$ are all compatible sum nodes $s_{ij}$ 
for each child 
$i$ of $j$,
the constant node $a_j$ 
with value $\psi_j(\mb{c}_j)$, and all indicator nodes 
$\phi_{vt}$ such that $X_v \in \Cb_j$, $t \compat \mb{c}_j$, and $X_v \notin \Cb_k$
for $k$ any node closer to the root of $\mc{T}$ than $j$.
The children of sum node $s_{ij}$ are the 
compatible product nodes $c_i$ 
of child $i$ of $j$ 
connected by separator $\mb{S}_{ij}$.
\end{define}
\noindent 
If $S$ is a tree-like SPF with junction tree $(T,Q)$, 
then it is not difficult to see both that $S$
is decomposable, since the indicators for each variable all appear at the same level,
and that each sum node $s_{jk}$ computes
$S_{s_{jk}}(\mb{S}_{jk}) =
\bigoplus_{( \mb{c} \in \X_{\Cb_j} ) \compat \mb{s}_{jk} } \psi_j( \mb{c} )
        \otimes  [\Cb_j = \mb{c}]
	\otimes \big( \bigotimes_{i \in \text{Ch}( j ) } 
	        S_{s_{ij}}( \mb{c}_{ \mb{S}_{ij}} )
		\big)$,
where the indicator children of 
$c_j$ have been combined into 
$[\Cb_j = \mb{c}]$, $\text{Ch}(j)$ are the children of $j$, 
and $i, j,k \in T$ with $j$ the child of $k$.
Further, $S(\x) = \bigotimes_{i \in T} \psi_i(\x_{\Cb_i})$ for any $\x \in \X$.
Thus, tree-like SPFs provide a method for decomposing an SPF.
For a tree-like SPF to be compatible with an SPF $F$, it
cannot assert independencies that do not hold in $F$.
\looseness=-1

\begin{define}
Let $F(\mb{U})$ be an SPF over variables $\mb{U}$ with pairwise-disjoint
subsets $\Xb, \Yb, \mb{W} \subseteq \mb{U}$. 
Then $\Xb$ and $\Yb$ are \emph{conditionally independent in $F$} given $\mb{W}$ 
iff $F(\Xb, \Yb, \mb{w} ) =  F(\Xb, \mb{w}) \otimes F(\Yb, \mb{w})$ for all $\mb{w} \in \mc{W}$, where
$F(\Xb) = \bigoplus_{\mc{Y}} F(\Xb, \Yb)$ for $\{\Xb,\Yb\}$ a partition of $\mb{U}$.
\end{define}
\vspace{-0.1em}

Similarly, a junction tree $\mc{T} = (T,Q)$ is incompatible with $F$ if it 
asserts independencies that are not in $F$, 
where variables $X$ and $Y$ are conditionally independent in $\mc{T}$ 
given $\mb{W}$ if $\mb{W}$ separates $X$ from $Y$. A set of variables $\mb{W}$ \emph{separates}
$X$ and $Y$ in $\mc{T}$ iff after removing all vertices
$\{ i \in T : \Cb_i \subseteq \mb{W} \}$ from $T$ there is 
no pair of vertices $i,j \in T$ such
that $X \in \Cb_i$, $Y \in \Cb_j$, and $i,j$ are connected.
\looseness=-1
\vspace{-0.15em}

Inference complexity is commonly parameterized by \emph{treewidth},
defined for a junction tree $\mc{T} = (T,Q)$ as the size of the largest cluster minus one;
i.e., $tw(\mc{T}) = \max_{i \in T} |\Cb_i| - 1$.
The treewidth of an SPF $S$
is the minimum treewidth over all junction trees compatible with $S$.
Notice that
these definitions of junction tree and treewidth reduce to the standard
ones~\citep{Kask2005}.
If the treewidth of $S$ is bounded then inference in $S$ is efficient because
there must exist a compatible tree-like SPF that has bounded treewidth.
Note that the trivial junction tree with only a single cluster 
is compatible with every SPF.

\begin{corollary}
Every SPF with bounded treewidth can be summed in time linear in the cardinality of its scope.
\label{cor:JTisSPF}
\end{corollary}
\vspace{-0.2em}

Due to space limitations, all other proofs are provided in Appendix~\ref{sec:proofs-sup}. 
For any SPF, tree-like SPFs are just one type of compatible SPF,
one with size exponential in treewidth;
however, there are many other compatible SPFs.
In fact, there can be compatible (decomposable) SPFs that are
exponentially smaller than any compatible tree-like SPF.

\begin{corollary}
Not every SPF that can be summed in time linear in the cardinality of its scope has bounded treewidth.
\vspace{-\baselineskip}
\label{cor:SPFnotJT}
\end{corollary}
\vspace{-0.2em}

Given existing work on tractable high-treewidth inference, it is
perhaps surprising that the above results do not exist in the
literature at this level of generality.
Most relevant is the preliminary work of \citet{Kimmig2012}, which proposes
a semiring generalization of arithmetic circuits for
knowledge compilation and does not address learning. 
Their main results show that summation of 
circuits that are both decomposable and either deterministic or based on an idempotent
sum takes time linear in their size,
whereas we show that decomposability alone is sufficient, a much weaker condition. 
In fact, over the same set of variables,
deterministic circuits may be exponentially larger and are never smaller than
non-deterministic circuits~\citep{Darwiche2002,Kimmig2012}. 
We note that while decomposable circuits can be made deterministic by introducing
hidden variables, this does not imply that these properties are equivalent.
\vspace{-0.2em}

%
%
%
%
%
Even when restricted to specific semirings, such as those for logical and
probabilistic inference (e.g., \citet{Darwiche2001,Darwiche2003,Poon2011}), 
some of our results have not previously been shown formally,
although some have been foreshadowed informally. 
Further, existing semiring-specific results (discussed further below)
do not make it clear that the semiring properties are all
that is required for tractable high-treewidth inference.
Our results are thus simpler and more general.
Further, the sum-product theorem
provides the basis for general algorithms for
inference in arbitrary SPFs (Section~\ref{sec:inference})
and for learning tractable high-treewidth representations 
(i.e., decomposable SPFs) in any semiring (Section~\ref{sec:learning}).


\vspace{-0.6em}
\section{Inference in non-decomposable SPFs}
\label{sec:inference}

Inference in arbitrary SPFs can be performed
in a variety of ways, some more efficient than others.
We present an algorithm for summing an SPF that adapts to the structure of the SPF
and can thus take exponentially less time
than constructing and summing a compatible tree-like SPF~\citep{Bacchus2009}, 
which imposes a uniform decomposition structure.
SPF $S$ with root node $r$ is summed by 
calling \Call{SumSPF}{$r$}, for which
pseudocode is shown in Algorithm~\ref{alg:sumspf}. \alg{SumSPF}
is a simple recursive algorithm for summing an SPF 
(note the similarity between its structure and the proof of the sum-product theorem).
If $S$ is decomposable, then \alg{SumSPF} simply recurses to the 
bottom of $S$, sums the leaf functions, and evaluates $S$ in an upward pass.
If $S$ is not decomposable, \alg{SumSPF}
decomposes each product node it encounters while summing $S$.
\looseness=-1

\newcommand*\Let[2]{#1 $\gets$ #2}
\algrenewcommand\algorithmicrequire{\textbf{Input:}}
\algrenewcommand\algorithmicensure{\textbf{Output:}}
\renewcommand{\algorithmiccomment}[1]{\bgroup\hfill\footnotesize //~\emph{{#1}}\egroup}

\algnewcommand{\IIf}[1]{\State\algorithmicif\ #1\ \algorithmicthen}
\algnewcommand{\ElseI}[1]{\State\algorithmicelse\ #1\ }
\algnewcommand{\EndIIf}{\unskip\ }

\algrenewcommand\algorithmicindent{0.8em}%

\newcommand{\ev}{\emph{sum}}

\makeatletter
\newcommand{\mybox}[2][\fboxsep]{{%
  \setlength{\fboxsep}{#1}\fbox{\m@th$\displaystyle#2$}}}
\makeatother

\begin{algorithm}[tbh]
  \caption{Sum an SPF.}
  \label{alg:sumspf} 
  \begin{algorithmic}[1]
    \Require{node $v$, the root of the sub-SPF $S_v(\Xb_v)$ }
    \Ensure{\ev, which is equal to $\bigoplus_{\vb \in \X_v} S_v(\vb)$ } 
    \vspace{0.2em}
    \Function{\sumspf}{$v$}
      \IIf{$\tuple{v,\text{\ev}}$ in cache} \Return{\ev} \EndIIf
      \If{$v$ is a sum node} \Comment{$\Xb_{v \backslash c} = \Xb_v \backslash \Xb_c$}
        \State \Let{\ev}{$\bigoplus_{c \in \text{Ch}(v)} \Call{\sumspf}{c} \otimes 
               \bigoplus_{\X_{v \backslash c}}1$} \label{alg1:sum}
        \vspace{-0.2em}
      \ElsIf{$v$ is a product node}
        \If{$v$ is decomposable} 
          \State \Let{\ev}{$\bigotimes_{c \in \text{Ch}(v)}$ \Call{\sumspf}{$c$}}
           \label{alg1:prod}
        \Else
          \State \Let{\ev}{\Call{\sumspf}{~\alg{Decompose}($v$)~}}
        \EndIf
        \vspace{-0.2em}
      \Else \Comment{$v$ is a leaf with constant $a$ or function $\phi_v$}
        \IIf{$v$ is a constant} \Let{\ev}{$a$}
        \ElseI \ev $\gets$ $\bigoplus_{x_j \in \X_j} \phi_v(x_j)$ 
        \EndIIf
      \EndIf
      \vspace{-0.4em}
      \State cache $\tuple{v, \text{\ev}}$
      \State\Return{\ev}
      \vspace{0.4em}
    \EndFunction
  \end{algorithmic}
\end{algorithm}
%
%

\begin{algorithm}[tbh]
    \caption{Decompose a product node.}
\label{alg:decompose}
  \begin{algorithmic}[1]
    \Require{product node $v$, with children $\{c\}$}
    \Ensure{node $s$, such that its children are decomposable with respect to $X_t$ and $S_s, S_v$ are compatible}
    \Function{Decompose}{$v$} 
    \State \Let{$X_t$}{choose var. that appears in multiple $\Xb_c$}
    \State \Let{$\Xb_{v \backslash t}$}{$\Xb_v \backslash \{ X_t \}$}
    \State \Let{$s$}{create new sum node}
    \ForAll{$x_{i} \in \X_t$}       
      \State create simplified $S_{v_i}(\Xb_{v \backslash t}) \gets S_v(\Xb_{v \backslash t}, x_{i})$
      \State set $v_i$ as child of $s$ \Comment{$v_i$ is the root of $S_{v_i}$}
      \State set $f(X_t) = [X_t = x_{i}]$ as child of $v_i$
    \EndFor
    \State set $s$ as a child of each of $v$'s parents
    \State remove $v$ and all edges containing $v$
    \State \Return $s$
    \vspace{0.4em}
    \EndFunction
  \end{algorithmic}
\end{algorithm}

Decomposition can be achieved in many different ways, but we base our method
on a common algorithmic pattern that already occurs in many of the inference problems we consider,
resulting in
a general, semiring-independent algorithm for summing any SPF.
%
%
\Call{Decompose}{}, shown in Algorithm~\ref{alg:decompose}, 
chooses a variable $X_t$ that
appears in the scope of multiple of $v$'s children;
creates $|\X_t|$ partially assigned and simplified copies $S_{v_i}$ of 
the sub-SPF $S_v$ for $X_t$ assigned to each value $x_i \in \X_t$;
multiplies each $S_{v_i}$ by an
indicator to ensure that only one is ever non-zero when $S$ is evaluated;
and then replaces $v$ with a sum over $\{ S_{v_i} \}$.
Any node $u \in S_v$ that does not have $X_t$ in its scope is re-used
across each $S_{v_i}$, which can drastically limit the amount of duplication that occurs. 
Furthermore, each $S_{v_i}$ is simplified by removing any nodes that became $0$
when setting $X_t = x_i$. 
Variables are chosen heuristically; a good heuristic minimizes the amount of duplication
that occurs. 
Similarly, \Call{SumSPF}{} heuristically orders the children in 
lines~\ref{alg1:sum} and~\ref{alg1:prod}. A good ordering will first
evaluate children that may return an absorbing value
(e.g., $0$ for $\otimes$) 
because
\Call{SumSPF}{} can break out of these lines if this occurs.
In general, decomposing
an SPF is hard, and the resulting decomposed SPF may be exponentially larger
than the input SPF, although good heuristics can often avoid this.
Many extensions to \alg{SumSPF} are also possible, some of which we detail in later sections. 
Understanding inference in non-decomposable SPFs
is important for future work on extending SPF learning to even more challenging
classes of functions, particularly those without obvious 
decomposability structure.


\vspace{-0.6em}
\section{Learning tractable representations}
\label{sec:learning}

\newcommand{\xbi}{\x^{(i)}}
\newcommand{\yii}{y^{(i)}}
\newcommand{\ybi}{\y^{(i)}}

Instead of performing inference in an intractable model, it can often be simpler to 
learn a tractable representation directly from data (e.g.,~\citet{Bach2001,Gens2013}).
The general problem we consider 
is that of learning 
a decomposable SPF $S : \X \rightarrow R$ on a semiring $(R, \oplus, \otimes, 0, 1)$
from a set of i.i.d. instances
$T = \{ (\xbi, \yii)\}$ drawn from a fixed distribution ${D}_{\X \times R}$, where
$\yii = \bigoplus_{\Z} F(\xbi, \Zb)$, $F$ is some (unknown) SPF, and
$\Zb$ is a (possibly empty) set of unobserved variables or parameters, such that 
$S(\xbi) \approx \yii$, for all $i$. After learning, $\bigoplus_{\X} S(\Xb)$ can be
computed efficiently.
In the sum-product semiring, this corresponds to summation (or integration), for which
estimation of a joint probability distribution over $\Xb$ is a special case.

For certain problems, such as constraint satisfaction or MPE inference, the desired
quantity is the argument of the sum. 
This can be recovered (if meaningful in the current semiring) from an SPF by
a single downward pass that recursively selects all children of a
product node and the (or a) active child of a sum node (e.g., the child with the smallest
value if minimizing). 
Learning for this domain corresponds to a generalization of learning
for structured prediction~\citep{Taskar2005}. Formally,
the problem is to learn an SPF $S$ from 
instances $T = \{ (\xbi, \ybi)\}$, where $\ybi = \argoplus_{\y \in \Y} F(\xbi, \y)$, such that
$\argoplus_{\y \in \Y} S(\xbi, \y) \approx \ybi$, for all $i$. Here, $\xbi$ can be an arbitrarily
structured input and inference is over the variables $\Yb$.
Both of the above learning problems can be solved by the algorithm schema we present,
with minor differences in the subroutines. We focus here on the former but
discuss the latter below, alongside experiments on learning nonconvex functions that, by 
construction, can be efficiently optimized.
\looseness=-1
\vspace{-0.2em}

As shown by the sum-product theorem, the key to tractable inference is to identify
the decomposability structure of an SPF.
The difficulty, however, is that in general this structure varies throughout the space.
For example, as a protein folds there exist conformations of the protein in
which two particular amino acids are energetically independent (decomposable), and
other conformations in which these amino acids directly interact, but in which other amino
acids may no longer interact.
This suggests a simple algorithm, which we call \alg{LearnSPF} (shown in Algorithm~\ref{alg:learnspf}), 
that first tries to identify a decomposable partition of the variables and, 
if successful, recurses on each subset of variables in order to 
find finer-grained decomposability.
Otherwise, \alg{LearnSPF} clusters the training
instances, grouping those with analogous decomposition structure, and recurses
on each cluster.
Once either the set of variables is small enough to be summed
over (in practice, unary leaf nodes are rarely necessary) or the number
of instances is too small to contain meaningful statistical information, \alg{LearnSPF}
simply estimates an SPF $S(\Xb)$ such that $S(\xbi) \approx \yii$ for all $i$ in
the current set of instances.
\alg{LearnSPF} is a generalization of LearnSPN~\citep{Gens2013}, a simple but 
effective SPN structure learning algorithm. 

\begin{algorithm}[t]
  \caption{Learn a decomposable SPF from data.}
  \label{alg:learnspf} 
  \begin{algorithmic}[1]
    \Require{a dataset $T = \{(\xbi, \yii)\}$ over variables $\Xb$}
    \Require{integer thresholds $t, v > 0$}
    \Ensure{$S(\Xb)$, an SPF over the input variables $\Xb$} 
    \Function{LearnSPF}{$T, \Xb$}
      \If{$|T| \leq t$ or $|\Xb| \leq v$}
        \State estimate $S(\Xb)$ such that $S(\xbi) \approx \yii$ for all $i$
      \Else
      \State decompose $\Xb$ into disjoint subsets $\{\Xb_i\}$ 
      \If{$|\{\Xb_i\}| > 1$}
        \State \Let{$S(\Xb)$}{$\bigotimes_{i} \Call{LearnSPF}{T, \Xb_i$}}
      \Else
        \State cluster $T$ into subsets of similar instances $\{T_j\}$
        \State \Let{$S(\Xb)$}{$\bigoplus_{j} \Call{LearnSPF}{T_j, \Xb$}}
      \EndIf
      \EndIf
      \vspace{-0.4em}
      \State \Return $S(\Xb)$
    \EndFunction
  \end{algorithmic}
\end{algorithm}

\alg{LearnSPF} is actually an algorithm schema 
that can be instantiated with different
variable partitioning, clustering, and leaf creation subroutines 
for different
semirings and problems.
To successfully decompose the variables, \alg{LearnSPF} must find a partition  
$\{\Xb_1, \Xb_2\}$ of the variables $\Xb$ such that 
$\bigoplus_{\X}S(\Xb) \approx (\bigoplus_{\X_1} S_1(\Xb_1)) \otimes (\bigoplus_{\X_2} S_2(\Xb_2))$.
We refer to this as \emph{approximate} decomposability. In probabilistic inference,
mutual information or pairwise independence tests can be used to determine
decomposability~\citep{Gens2013}. For our experiments,
decomposable partitions correspond to the connected components of a graph over 
the variables in which correlated variables are connected.
%
Instances can be clustered by virtually any clustering algorithm, such as
a naive Bayes mixture model or $k$-means, which we use in 
our experiments. Instances can also be split by conditioning on specific values
of the variables, as in \alg{SumSPF} or
in a decision tree.
Similarly, leaf functions can be estimated using any appropriate learning algorithm,
such as linear regression or kernel density estimation.
\vspace{-0.1em}

In Section~\ref{sec:experiments}, we present preliminary experiments on
learning nonconvex functions that can be globally optimized in polynomial time. 
However, this is just one particular application of \alg{LearnSPF}, which can be
used to learn a tractable representation for any problem that consists
of summation over a semiring.
In the following section, we briefly discuss common inference problems
that correspond to summing an SPF on a specific semiring. 
For each, we demonstrate the benefit of the sum-product theorem, relate its core
algorithms to \alg{SumSPF}, and specify the problem solved by \alg{LearnSPF}.
Additional details and semirings can be found in the Appendix.
Table~\ref{table:spfs} provides a summary of some of the relevant inference problems.
\looseness=-1
\newcommand{\myITwidth}{1.8cm}
\newcommand{\myVarWidth}{1.9cm}
\newcommand{\myLeafWidth}{3.2cm}
\newcommand{\myAlgWidth}{2.0cm}
\begin{table*}[tbh!]
\centering
\vspace{-1mm}
\caption{\small Some of the inference problems that correspond to summing an SPF on a specific semiring,
with 
details on the variables and leaf
functions and a core algorithm that is an instance of \alg{SumSPF}.
$\B = \{0,1\}$.
$\N$ and $\R$ denote the natural and real numbers. 
Subscript $+$ denotes the restriction to non-negative numbers
and subscript $(-)\infty$ denotes the inclusion of (negative) $\infty$.
$\mb{U}_m$ denotes the universe of relations of arity up to $m$ 
(see Appendix~\ref{sec:relational-sup}). 
$\N[\Xb]$ denotes the polynomials with coefficients from $\N$.
See the Appendix for information on MPE-SAT~\citep{Sang2007} and Generic-Join~\citep{Ngo2014}.
} \label{table:semirings}
\begin{tabular}{ l  l p{3.4cm}  l  l  l }
\toprule
\textbf{Domain} & \textbf{Inference task} & \textbf{Semiring} & \textbf{Variables} & \textbf{Leaf functions} & \textbf{\textsc{SumSPF}}\\ 
\toprule
\multirow{2}{\myITwidth}{Logical \\~inference} & SAT & $(\B, \vee, \wedge, 0, 1)$ &  \multirow{1}{\myVarWidth}{Boolean} & \multirow{1}{\myLeafWidth}{Literals} & DPLL \\
& \#SAT & $(\N, +, \times, 0, 1)$ & Boolean & Literals & \#DPLL \\ 
& MAX-SAT & $(\N_{-\infty}, \max, +, -\infty, 0)$ & Boolean & \multirow{1}{\myLeafWidth}{Literals}  & MPE-SAT \\
\midrule
\multirow{2}{\myITwidth}{Constraint \\~satisfaction} & CSPs & $(\B, \vee, \wedge, 0, 1)$ & \multirow{1}{\myVarWidth}{Discrete} & \multirow{1}{\myLeafWidth}{Univariate constraints} & \multirow{1}{\myAlgWidth}{Backtracking} \\
& Fuzzy CSPs & $([0,1], \max, \min, 0, 1)$ & Discrete  & Univariate constraints & ~~~~~~- \\
& Weighted CSPs & $(\R_{+,\infty}, \min, +, \infty, 0)$ & Discrete & Univariate constraints & ~~~~~~- \\ 
\midrule
\multirow{2}{\myITwidth}{Probabilistic \\~inference} & Marginal & $(\R_+, +, \times, 0, 1)$ & \multirow{1}{\myVarWidth}{Discrete} & Potentials & \multirow{1}{\myAlgWidth}{Recursive \\~conditioning} \\
& MPE & $(\R_+, \max, \times, 0, 1)$ & Discrete & Potentials  &  \\ 
\midrule
\multirow{2}{\myITwidth}{{Continuous \\~functions}} & Integration & $(\R_+, +, \times, 0, 1)$ & Continuous & \multirow{1}{\myLeafWidth}{Univariate functions} & ~~~~~~- \\
& Optimization & $(\R_{\infty}, \min, +, \infty, 0)$ & Continuous & Univariate functions & RDIS \\
\midrule
\multirow{2}{\myITwidth}{{Relational \\~databases}} 
& Unions of CQs & $(\mb{U}_m, \cup, \bowtie, \varnothing, \JI)$ & \multirow{1}{\myVarWidth}{Sets of tuples} & \multirow{1}{\myLeafWidth}{Unary tuples}  & Generic-Join \\
& Provenance & $(\N[\Xb], +, \times, 0, 1)$ & Discrete & $K$-relation tuples  & ~~~~~~- \\
\bottomrule
\label{table:spfs}
\end{tabular}
\vspace{-2.5em}
\end{table*}

\vspace{-0.6em}
\section{Applications to specific semirings}
\vspace{-0.35em}
\subsection{Logical inference}
\label{sec:sat}


\vspace{-0.3em}
Consider the Boolean semiring $\mc{B} = (\B, \vee, \wedge,0,1)$, where $\B = \{0,1\}$, $\vee$ is 
logical disjunction (OR), and $\wedge$ is logical conjunction (AND). 
If each variable is Boolean and leaf functions are literals 
(i.e., each $\phi_l(X_j)$ is $X_j$ or $\neg X_j$, where $\neg$ is logical negation),
then SPFs on $\mc{B}$ correspond exactly to negation normal form (NNF),
a DAG-based representation of a propositional formula (sentence)~\citep{Barwise1982}.
An NNF can be exponentially smaller than the same sentence in a standard (flat)
representation such as conjunctive or disjunctive normal form (CNF or DNF),
and is never larger~\citep{Darwiche2002}.
Summation of an NNF $F(\Xb)$ on $\mc{B}$ is $\bigvee_{\X} F(\Xb)$,
which corresponds to propositional satisfiability (SAT):
the problem of determining if there exists 
a satisfying assignment for $F$.
Thus, the tractability of SAT for decomposable NNFs
follows directly from the sum-product theorem.

\setcounter{citedcorollary}{2}
\addtocounter{corollary}{1}
\begin{citedcorollary}[\citep{Darwiche2001}]
The satisfiability of a decomposable NNF is decidable in time linear in its size.
\label{cor:sat}
\end{citedcorollary}
\vspace{-0.4em}

Satisfiability of an arbitrary NNF can be determined either by decomposing the NNF
or by expanding it to a CNF and using a SAT solver. DPLL~\citep{Davis1962}, the standard algorithm
for solving SAT, is an instance of \alg{SumSPF} (see also \citet{Huang2007}).
Specifically, DPLL is a recursive algorithm that at each level
chooses a variable $X \in \Xb$ for CNF $F(\Xb)$ and
computes $F = F|_{X=0} \vee F|_{X=1}$ by recursing on each disjunct, 
where $F|_{X=x}$ is $F$ with $X$ assigned value $x$.
Thus, each level of recursion of DPLL corresponds  to a call to \alg{Decompose}.

\vspace{-0.2em}

Learning in the Boolean semiring is a well-studied area, which includes problems
from learning Boolean circuits~\citep{Jukna2012} 
(of which decomposable SPFs are a restricted subclass, known as 
syntactically multilinear circuits) 
to learning sets of rules~\citep{Rivest1987}. However, learned rule sets are typically
encoded in large CNF knowledge bases, making reasoning over them intractable.
In contrast, decomposable NNF is a tractable but expressive formalism for knowledge
representation that supports a rich class of polynomial-time logical operations, 
including SAT~\citep{Darwiche2001}. 
Thus, \alg{LearnSPF} in this semiring
provides a method for learning large, complex knowledge bases that are encoded in decomposable NNF
and therefore support efficient querying, which could greatly benefit existing rule learning systems.

\vspace{-0.7em}
\subsection{Constraint satisfaction.}
\vspace{-0.2em}
A constraint satisfaction problem (CSP) consists of a set of constraints $\{C_i\}$
on variables $\Xb$, where each
constraint $C_i(\Xb_i)$ 
specifies the satisfying 
assignments to its variables.
Solving a CSP consists of finding an assignment to $\Xb$ that satisfies each constraint.
When constraints are functions $C_i : \X_i \rightarrow \B$ 
that are $1$ when $C_i$ is satisfied and $0$ otherwise, then
$F(\Xb) = \bigwedge_{i} C_i(\Xb_i) =
\bigwedge_{i} \bigvee_{\x_i \in \X_i} \big( C_i(\x_i) \wedge [\Xb_i = \x_i] \big) =
\bigwedge_{i} \bigvee_{\x_i \in \X_i} \big( C_i(\x_i) \wedge \bigwedge_{X_t \in \Xb_i} [X_t = \x_{it}] \big)$ 
is a CSP 
and $F$ is an SPF on the Boolean semiring $\mc{B}$, 
i.e., an OR-AND network (OAN), a generalization of NNF, 
and a decomposable CSP is one with a decomposable OAN.
Solving $F$ corresponds to computing $\bigvee_{\X} F(\Xb)$, which is summation
on $\mc{B}$ (see also \citet{Bistarelli1997,Chang2005,Rollon2013}). 
The solution for $F$ can be recovered with a downward pass that recursively
selects the (or a) non-zero child of an OR node, and all children of an AND node.
Corollary~\ref{cor:scsp} follows immediately.
\looseness=-1

\begin{corollary}
Every decomposable CSP can be solved in time linear in its size.
\label{cor:scsp}
\end{corollary}

\vspace{-0.3em}
Thus, for inference to be efficient it suffices that the CSP be expressible by a tractably-sized 
decomposable OAN; a much weaker condition than that of low treewidth.
%
Like DPLL, backtracking-based search algorithms~\citep{Kumar1992} for CSPs are
also instances of \alg{SumSPF}
(see also \citet{Mateescu2005}).
Further, SPFs on a number of other semirings correspond to various extensions
of CSPs, including fuzzy, probabilistic, and weighted CSPs (see
Table~\ref{table:spfs} and~\citet{Bistarelli1997}).

\alg{LearnSPF} for CSPs addresses a
variant of structured prediction~\citep{Taskar2005}; specifically,
learning a function $F : \Xb \rightarrow \B$ such that
$\argOR_{\y} F(\xbi, \y) \approx \ybi$ for training data $\{(\xbi, \ybi)\}$,
where $\xbi$ is a structured object representing a CSP
and $\ybi$ is its solution. 
\alg{LearnSPF} solves this problem
while guaranteeing that the learned CSP remains tractable.
This is a much simpler and more attractive approach than existing 
constraint learning methods such as \citet{Lallouet2010}, which uses inductive logic
programming and has no tractability guarantees.
\looseness=-1

\vspace{-0.6em}
\subsection{Probabilistic inference}
\label{sec:prob}
\vspace{-0.2em}

Many probability distributions can be compactly represented as graphical
models: $P(\Xb) = \frac{1}{Z} \prod_i \psi_i(\Xb_i)$, where $\psi_i$
is a potential over variables $\Xb_i \subseteq \Xb$
and $Z$ is 
known as the partition function~\citep{Pearl1988}. One of
the main inference problems in graphical models is to compute the probability
of evidence $\mb{e} \in \X_E$ for variables $\Xb_E \subseteq \Xb$, 
$P(\mb{e}) = \sum_{\X_{\ol{E}}} P(\mb{e}, \Xb_{\ol{E}})$, where $\Xb_{\ol{E}} = \Xb
\backslash \Xb_E$. The partition function $Z$ is the unnormalized probability
of empty evidence ($\Xb_E = \varnothing$). 
Unfortunately, computing $Z$ or $P(\mb{e})$ is generally intractable.
Building on a number of earlier works~\citep{Darwiche2003,Dechter2007,Bacchus2009},
\citet{Poon2011}
introduced sum-product networks (SPNs), a class of distributions in which
inference is guaranteed to be tractable. 
An SPN is an SPF on the non-negative real sum-product
semiring $(\R_+, +, \times, 0, 1)$. A graphical model is a flat SPN, in the
same way that a CNF is a flat NNF~\citep{Darwiche2002}.
For an SPN $S$, the unnormalized probability of evidence $\mb{e} \in \X_E$
for variables $\Xb_E$ is computed by replacing each leaf function
$\phi_l \in \{\phi_l(X_j) \in S | X_j \in \Xb_E \}$ with the constant
$\phi_l(\mb{e}_j)$ and summing the SPN. The corollary below follows
immediately from the sum-product theorem.

\begin{corollary}
The probability of evidence in a decomposable SPN can be computed in time linear in its size.
\label{cor:probinf}
\end{corollary}
\vspace{-0.2em}

A similar result (shown in Appendix~\ref{sec:prob-sup}) for finding the most probable state of the 
non-evidence variables also follows from the sum-product theorem. 
One important consequence of the sum-product theorem 
is that decomposability is the sole condition required for
an SPN to be tractable; previously, completeness was also required \citep{Poon2011,Gens2013}.
This 
expands the range of tractable SPNs and simplifies the
design of tractable representations based on them, such as tractable
probabilistic knowledge bases~\citep{Domingos2012}.
\looseness=-1
\vspace{-0.1em}

Most existing algorithms for inference in graphical models correspond to different
methods of decomposing a flat SPN, and can be loosely clustered into 
tree-based, conditioning, and compilation methods, all of which \alg{SumSPF} generalizes.
Details are provided in Appendix~\ref{sec:prob-sup}.
\vspace{-0.3em}

\alg{LearnSPF} for SPNs corresponds to learning a probability distribution
from a set of samples $\{(\xbi, \yii)\}$. Note that $\yii$ in this
case is defined implicitly by the empirical frequency of $\xbi$ in the dataset.
Learning the parameters and structure of SPNs is a fast-growing area of research
(e.g.,~\citet{Gens2013,Rooshenas2014,Peharz2014,Adel2015}), 
and we refer readers to these references for more details. 

\vspace{-0.75em}
\subsection{Integration and optimization}
\label{sec:continuous}
\vspace{-0.35em}

SPFs can be generalized to continuous (real) domains,
where each variable $X_i$ has domain $\X_i \subseteq \R$ and the semiring set
is a subset of $\R_\infty$.
For the sum-product theorem to hold, the only additional conditions are that
(C1) $\bigoplus_{X_j} \phi_l(X_j)$ is computable in constant time for all leaf functions,
and (C2) $\bigoplus_{\X_{v \backslash c}} 1 \neq \infty$ 
for all sum nodes $v \in S$ and all children $c \in \text{Ch}(v)$, 
where $\X_{v \backslash c}$ is the domain of $\Xb_{v \backslash c} = \Xb_v \backslash \Xb_c$.
\vspace{-0.3em}

\noindent \textbf{Integration.}
In the non-negative real sum-product semiring $(\R_+, +, \times, 0, 1)$, summation of an SPF with
continuous variables 
corresponds to integration over $\X$. Accordingly, we generalize SPFs as follows.
Let $\mu_1, \dots, \mu_n$ be measures over $\X_1, \dots, \X_n$, respectively, where
each leaf function $\phi_l : \X_j \rightarrow \R_+$ is integrable with respect to $\mu_j$, which satisfies (C1).
Summation (integration) of an SPF $S(\Xb)$ then corresponds to computing 
$\int_{\X} S(\Xb) d\mu = \int_{\X_1} \cdots \int_{\X_n} S(\Xb) d\mu_1 \cdots d\mu_n$.
For (C2), 
$\bigoplus_{\X_{v \backslash c}} 1 = \int_{\X_{v \backslash c}} 1 ~d\mu_{v \backslash c}$
must be integrable for all sum nodes $v \in S$ and all children $c \in \text{Ch}(v)$, 
where $d\mu_{v \backslash c} = \prod_{ \{j : X_j \in \Xb_{v \backslash c} \} } d\mu_j$.
We thus assume that either $\mu_{v \backslash c}$ has finite support 
over $\X_{v \backslash c}$ or that $\Xb_{v \backslash c} = \varnothing$.
Corollary~\ref{cor:integration} follows immediately.

\begin{corollary}
\label{cor:integration}
Every decomposable SPF of real variables can be integrated in time linear in its size.
\end{corollary}
\vspace{-0.4em}

Thus, decomposable SPFs define a class of functions for which exact integration is tractable.
%
%
\alg{SumSPF} defines a novel algorithm for (approximate) integration that
is based on recursive problem decomposition, and can be
exponentially more efficient than standard integration algorithms
such as trapezoidal or Monte Carlo methods~\citep{Press2007} because
it dynamically decomposes the problem at each recursion level
and caches intermediate computations. More detail 
is provided in Appendix~\ref{sec:continuous-sup}. 
\vspace{-0.2em}

In this semiring, \alg{LearnSPF} learns a
decomposable continuous SPF $S : \X \rightarrow \R_+$ 
on samples $\{(\xbi, \yii = F(\xbi))\}$ from an SPF $F : \X \rightarrow \R_+$,
where $S$ can be 
integrated efficiently over the domain $\X$. Thus, \alg{LearnSPF} provides a novel
method for learning and integrating complex functions, such as the partition
function of continuous probability distributions.
\vspace{-0.2em}


\noindent \textbf{Nonconvex optimization.}
Summing a continuous SPF in one of the 
min-sum, min-product, max-sum, or max-product semirings corresponds to
optimizing a (potentially nonconvex) continuous objective function. 
Our results hold for all of these, but we focus here on the 
real min-sum semiring $(\R_\infty, \min, +, \infty, 0)$, where summation of
a min-sum function (MSF) $F(\Xb)$ corresponds to computing $\min_{\X} F(\Xb)$.
A flat MSF is simply a sum of terms.
To satisfy (C1), we assume that $\min_{x_j \in \X_j} \phi_l(x_j)$ is 	
computable in constant time for all $\phi_l \in F$. 
(C2) is trivially satisfied for $\min$.
The corollary below follows immediately.

\begin{corollary}
The global minimum of a decomposable MSF can be found
in time linear in its size.
\label{cor:ncopt}
\end{corollary}
\vspace{-0.2em}

\alg{SumSPF} provides an outline for a general nonconvex optimization algorithm
for sum-of-terms (or product-of-factors) functions. 
The recent RDIS algorithm for nonconvex optimization~\citep{Friesen2015}, 
which achieves exponential speedups compared to other algorithms, is an
instance of \alg{SumSPF} where values are chosen via multi-start gradient descent
and variables in \alg{Decompose} are chosen by graph partitioning. 
\citet{Friesen2015}, however, do not specify tractability conditions for the optimization; 
thus, Corollary~\ref{cor:ncopt} defines a novel class of functions
that can be efficiently globally optimized.
\vspace{-0.1em}

For nonconvex optimization, \alg{LearnSPF} solves a variant of structured 
prediction~\citep{Taskar2005}, in which the 
variables to predict are continuous instead of discrete (e.g., protein folding, 
structure from motion~\citep{Friesen2015}).
The training data is a set $\{(\xbi, \ybi)\}$,
where $\xbi$ is a structured object representing a nonconvex function and
$\ybi$ is a vector of values specifying the global minimum of that function.
\alg{LearnSPF} learns a function $S : \X \times \Y \rightarrow \R_{\infty}$
such that $\argmin_{\y \in \Y} S(\xbi, \y) \approx \ybi$, where the $\argmin$
can be computed efficiently because S is decomposable. 
More detail is provided in Section~\ref{sec:experiments}.


\vspace{-0.6em}
\section{Experiments}
\label{sec:experiments}
\vspace{-0.1em}

We evaluated \alg{LearnSPF} on the task of learning a nonconvex decomposable
min-sum function (MSF) from a training set of solutions of instances of a highly-multimodal
test function consisting of a sum of terms. By learning an MSF, instead of just a sum of terms, 
we learn the general mathematical form of the optimization problem in such a way 
that the resulting learned problem is tractable, whereas the original sum of terms is not.
The test function we learn from is a variant of the Rastrigin 
function~\citep{Torn1989}, 
a standard highly-multimodal test function for global optimization consisting
of a sum of multi-dimensional sinusoids in quadratic basins. 
The function, $F_{\Xb}(\Yb) = F(\Yb; \Xb)$, has parameters $\Xb$, 
which determine the dependencies
between the variables $\Yb$ and the location of the minima.
To test \alg{LearnSPF}, we sampled a dataset of function instances 
$T = \{ (\xbi, \ybi ) \}_{i=1}^m$ from a
distribution over $\X \times \Y$, where 
$\ybi= \argmin_{\y \in \Y} F_{\xbi}(\y)$. 

\vspace{-0.1em}

\alg{LearnSPF} partitioned variables $\Yb$ based on
the connected components of a graph containing a node for each $Y_i \in \Yb$ and an
edge between two nodes only if $Y_i$ and $Y_j$ were correlated,
as measured by Spearman rank correlation.
Instances were clustered by running k-means on the values $\ybi$.
For this preliminary test, \alg{LearnSPF} did not learn the leaf functions
of the learned min-sum function (MSF) $M(\Yb)$; 
instead, when evaluating or minimizing a leaf node
in $M$, we evaluated or minimized the test function with all variables 
not in the scope of the leaf node fixed to $0$
(none of the optima were positioned at $0$).
This corresponds to having perfectly learned leaf nodes
if the scopes of the leaf nodes accurately reflect
the decomposability of $F$, otherwise a large error is incurred.
We did this to study the effect of learning the decomposability structure
in isolation from the error due to learning leaf nodes. The function used for comparison
is also perfectly learned. Thresholds
$t$ and $v$ were set to $30$ and $2$, respectively.


The dataset was split into $300$ training samples and $50$ test samples,
where $\min_{\Y} F_{\xbi}(\Yb) = 0$ for all $i$ for comparison purposes.
After training, we computed $\y_M = \argmin_{\Y} M(\Yb)$ 
for each function in the test set by 
first minimizing each leaf function (with respect to only those variables in the scope of the leaf function) 
with multi-start L-BFGS~\citep{Liu1989}
and then performing an upward and a downward pass in $M$.
Figure~\ref{fig:testfunc} shows the result of 
minimizing the learned MSF $M$ and evaluating the test function at $\y_M$ (blue line) compared to
running multi-start L-BFGS directly on the test function and reporting
the minimum found (red line),
where both optimizations are run for the same fixed amount of time (one minute per test sample).
\alg{LearnSPF} accurately learned the decomposition structure
of the test function, allowing it to find much better minima when optimized,
since optimizing many small functions at the leaves requires exploring exponentially
fewer modes than optimizing the full function.
Additional experimental details are provided in Appendix~\ref{sec:exper-sup}. 
\looseness=-1

\begin{figure}[tb]
\includegraphics[width=\columnwidth]{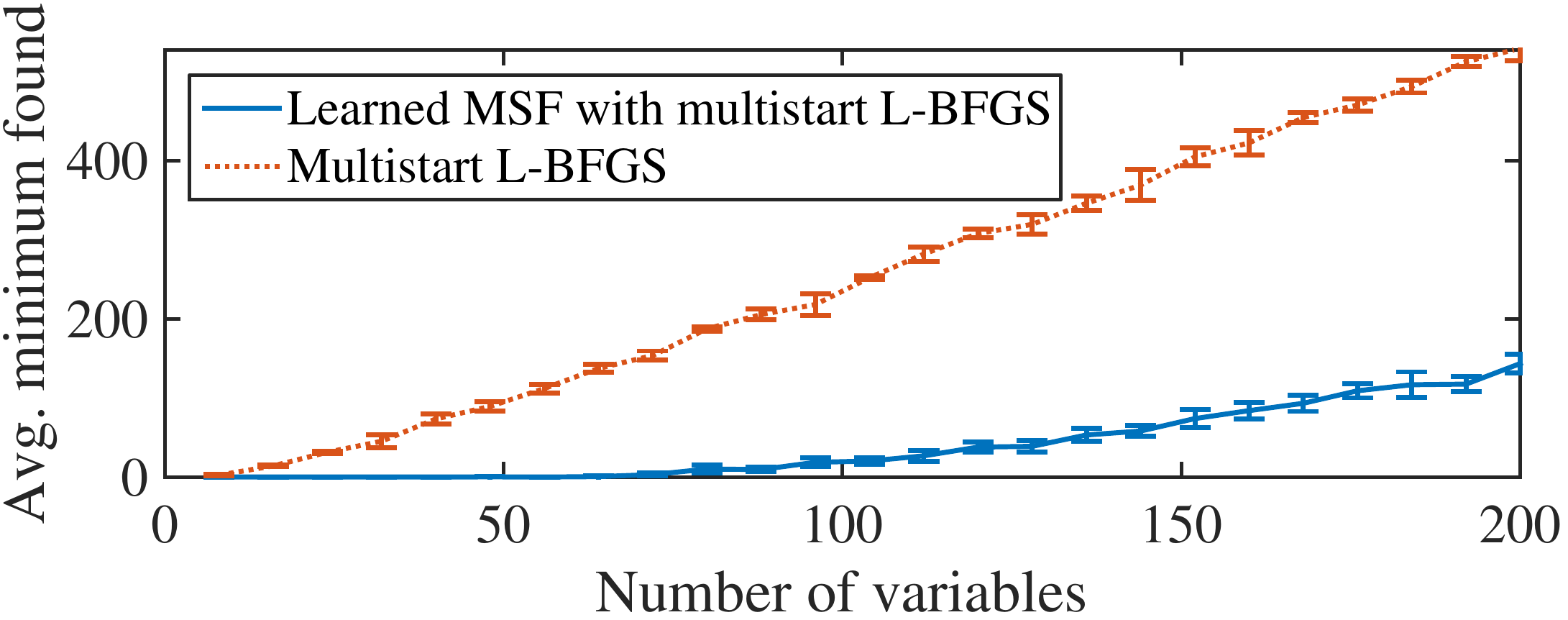}
\caption{\small{
The average minimum found over 20 samples of the test function 
versus the number of variables, with standard error bars. 
Each function was optimized for the same amount of time. 
\vspace{0.5em}
}}
\label{fig:testfunc}
\end{figure}

\vspace{-0.6em}
\section{Conclusion}
\label{sec:conclusion}

This paper developed a novel foundation for learning tractable representations
in any semiring based on the sum-product theorem,
a simple tractability condition for all inference problems that reduce
to summation on a semiring. With it, we developed a general inference algorithm 
and an algorithm for
learning tractable representations in any semiring. 
We demonstrated the power and generality of our approach by applying it
to learning a nonconvex function that can be optimized in polynomial time, 
a new type of structured prediction problem.
We showed empirically that our learned function greatly outperforms 
a continuous function learned without regard to the cost of optimizing it.
We also showed that the sum-product theorem specifies an exponentially
weaker condition for tractability than low treewidth and that its corollaries include
many previous results in
the literature, as well as a number of  novel results.
%
%
%

\vspace{-0.6em}
\section*{Acknowledgments}
This research was partly funded by
ONR grants N00014-13-1-0720 and N00014-12-1-0312, and AFRL contract
FA8750-13-2-0019. The views and conclusions contained in this
document are those of the authors and should not be interpreted as
necessarily representing the official policies, either expressed or
implied, of ONR, AFRL, or the United States Government.

\vspace{-0.6em}
\bibliographystyle{icml2016}
\bibliography{../../bibtex/library}

\newpage
\begin{appendices}

\setcounter{section}{0}
\renewcommand\thesection{\Alph{section}}

\section{Decomposable SPF summation complexity}
\label{sec:complexity-sup}

Let $S(\Xb)$ be a decomposable SPF with size $|S|$ on 
commutative semiring $(R, \oplus, \otimes, 0, 1)$, 
let $d = |\X_i|$ for all $X_i \in \Xb$ where $\Xb = (X_1, \dots, X_n)$,
and let the cost of $a \oplus b$ and $a \otimes b$ for any elements $a, b \in R$ be $c$.
Further, let $e$ denote the complexity of evaluating any unary leaf 
function $\phi_j(X_i)$ in $S$ and let 
${k = \max_{v \in S_{\text{sum}}, j \in \text{Ch}(v)} |\Xb_v \backslash \Xb_j| < n}$, 
where $S_{\text{sum}}, S_{\text{prod}}$, and $S_{\text{leaf}}$ are the
sum, product, and leaf nodes in $S$, respectively, and $\text{Ch}(v)$ are the children of $v$.
Then the complexity of computing $\bigoplus_{\x \in \X} S(\x)$ is 
${|S| \cdot c + 
|S_{\text{leaf}}| \cdot d(e + c) + 
|S_{\text{sum}}| \cdot (c + k d c)}$.

For certain simple SPFs that have very little internal structure and many input variables, the
worst case complexity of summing $S$ can be quadratic in $|S|$ and occurs in the rare and
restrictive case
where $k = O(n) = O(|S|)$, due to the
$\bigoplus_{\X_{v \backslash i}} 1$ term at each sum node (see proof of Theorem~\ref{thm:spt}). 
However, in any semiring with an 
idempotent sum (i.e., $a \oplus a = a$ for every $a \in R$) such as the min-sum or max-product semirings,
this term is always equal to $1$ and thus no computation is necessary. 
Alternatively, if the semiring supports multiplication and division as in the sum-product semiring
then this complexity can be reduced by first computing the product over all variables and then dividing 
out as needed. If the semiring has neither of these properties, these identity summations can still be
computed with a single preprocessing pass through the SPF since they are constants and independent of the 
input variables. For all semirings we've studied, this quadratic cost does not occur, 
but we include it for completeness.


\section{Logical inference (continued)}
\label{sec:logical-sup}

\textbf{Model counting.}
Model counting (\#SAT) is the problem of computing the number
of satisfying assignments of a Boolean formula.
The model count of an NNF $F$ can be obtained by \emph{translating} it from
the Boolean semiring to the counting sum-product semiring $\mc{P} = (\N,+,\times,0,1)$
($\R_+$ is used instead for weighted \#SAT), and then summing it.

\setcounter{define}{7}
\begin{define}
\emph{Translating} an SPF from semiring $(R,\oplus,\otimes,0,1)$ to semiring 
$(R',\boxplus,\boxtimes,0',1')$ with $R \subseteq R'$, involves
replacing each $\oplus$ node with a $\boxplus$ node, 
each $\otimes$ node with a $\boxtimes$ node,
and each leaf function that returns $0$ or $1$ with one that returns $0'$ or $1'$, respectively.
\end{define}

\noindent
However, simply summing the translated function $F'$ may
compute an incorrect model count because
the same satisfying assignment may be counted multiple times; this occurs when
the idempotence (a semiring $R$ is idempotent if $a \oplus a = a$ for $a \in R$) of the two semirings differs,
i.e., either semiring $R$ is idempotent and $R'$ is not, or vice versa.
If exactly one of the two semirings is idempotent, 
$F$ must be \emph{deterministic} to ensure that summing $F'$ gives the correct model count.

\begin{define}
An OR node is \emph{deterministic} iff the supports of its children are disjoint.
An NNF is deterministic iff all of its OR nodes are deterministic.
\end{define}

The support of a function $G(\Xb)$ is the set of points $\mc{S} \subseteq \X$ 
such that $G(\x) \neq 0$ for all $\x \in \mc{S}$.
If $F$ is deterministic and decomposable, then it follows from the sum-product theorem
that its model count can be computed efficiently. 
 

\setcounter{citedcorollary}{7}
\setcounter{corollary}{7}
\addtocounter{corollary}{1}
\begin{citedcorollary}[\citep{Darwiche2000a}]
The model count of a deterministic, decomposable NNF can be computed in time linear in its size.
\label{cor:sharpsat}
\end{citedcorollary}

Most algorithms for \#SAT (e.g., Relsat~\citep{BayardoJr2000}), 
Cachet~\citep{Sang2004}, \#DPLL~\citep{Bacchus2009}) 
are also instances of \alg{SumSPF}, since they extend DPLL by,
at each level of recursion, decomposing the CNF 
into independent components (i.e., no variable appears in multiple components), 
solving these separately, and caching the model count of each component.
Component decomposition corresponds to a decomposable
product node in \alg{SumSPF} and component caching corresponds to connecting
a sub-SPF to multiple parents. Notice that the sum nodes created by \alg{Decompose} 
are deterministic.

\textbf{MAX-SAT.}
MAX-SAT is the problem of computing the maximum number of satisfiable clauses of a CNF, 
over all assignments. It can be generalized to NNFs as follows.

\begin{define}
Let $F(\Xb)$ be an NNF and $\x \in \X$ an assignment. 
The \emph{SAT number} (SN) of a literal $\phi(X_j) \in F$ is $1$ 
if $\phi(\x_j)$ is true and 0 otherwise. 
The SN of an AND node is the sum of the SNs
of its children. The SN of an OR node is the max 
of the SNs of its children. 
\end{define}

MAX-SAT of an NNF $F(\Xb)$ is the problem of computing the 
maximum SAT number of the root of $F$ over all assignments $\x \in \X$. 
If $F$ is a CNF, then this reduces to standard MAX-SAT.
MAX-SAT of $F$ can be solved by translating $F$ to the max-sum semiring
$\mc{M} = (\N_{-\infty}, \max, +, -\infty, 0)$ (where $\R_{+,-\infty}$ is used for weighted MAX-SAT), 
and then summing it.
Clearly, $F'$ is an SPF on $\mc{M}$, i.e., a max-sum network.
The corollary below follows immediately from the sum-product theorem. 

\addtocounter{corollary}{1}
\begin{citedcorollary}[\citep{Darwiche2001}]
MAX-SAT of a decomposable NNF can be computed in time linear in its size.
\end{citedcorollary}

MAX-SAT of an arbitrary NNF (or CNF) can be computed by first translating
it to $\mc{M}$ and then calling \alg{SumSPF}, which can be
extended to perform
branch and bound (BnB)~\citep{Lawler1966} when traversing the SPF.
This allows \alg{SumSPF} to prune sub-SPFs that are not relevant to the final solution, 
which can greatly reduce the search space.
With this addition, DPLL-based BnB solvers for MAX-SAT 
(e.g.,~\citet{Heras2008} and references therein) are instances of \alg{SumSPF}.
Most relevant, however, is the MPE-SAT algorithm of~\citet{Sang2007}, since both
it and \alg{SumSPF} use decomposition and caching to improve their efficiency.

\section{Probabilistic inference (continued)}
\label{sec:prob-sup}

\textbf{Marginal inference (continued).}
Tree-based methods include junction-tree clustering~\citep{Lauritzen1988}
and variable elimination~\citep{Dechter1999},
which correspond (explicitly and implicitly, respectively) to
constructing a junction tree and then summing its corresponding
tree-like SPN. Conditioning algorithms such as
recursive conditioning~\citep{Darwiche2001b}, 
value elimination~\citep{Bacchus2002}, AND/OR search~\citep{Dechter2007},
and \#DPLL~\citep{Bacchus2009} 
traverse the space of partial assignments by recursively conditioning on variables and their values.
These algorithms vary in the flexibility of their variable ordering, decomposition,
and caching (see~\citet{Bacchus2009} for a comparison), but are all
instances of \alg{SumSPF}, which can use a fully-dynamic 
variable ordering, as value elimination can and \#DPLL does, or a fixed ordering, 
as in variants of recursive conditioning and AND/OR search. 
Decomposition and caching correspond to decomposable product nodes and 
connecting sub-SPNs to multiple parents, respectively, in \alg{SumSPF}.
Thirdly, inference in graphical models can be performed by compilation to
an arithmetic circuit (AC)~\citep{Darwiche2003}.
In discrete domains, \citet{Rooshenas2014} showed that SPNs and 
ACs are equivalent, but that SPNs are always smaller or equal in size. 
In continuous domains, however, it is unlikely that even this relationship exists,
because a AC would require an infinite number of indicator functions.
Furthermore, existing compilation methods require first encoding the
graphical model in very restrictive languages (such as CNF or SDDs),
which can make them exponentially slower than \alg{SumSPF}.
Finally, no tractability properties have been established
for ACs so there is no guarantee before compiling that inference will be tractable, 
nor have they been generalized to other semirings.

\textbf{MPE.}
Beyond computing the probability of evidence, 
another key probabilistic inference problem is finding the most probable or
MPE state of the non-evidence variables of $P(\Xb)$ given the evidence, 
$\argmax_{\X_{\ol{E}}} P(\mb{e}, \Xb_{\ol{E}})$
for evidence $\mb{e} \in \X_E$ where $\Xb_{\ol{E}} = \Xb \backslash \Xb_E$.
The MPE value (maximum probability
of any state) of an SPN $S$ can be computed by translating $S$ to the non-negative 
max-product semiring $(\R_+, \max, \times, 0, 1)$ and maximizing the
resulting SPF $S'$.  The MPE state can then be recovered by a downward pass
in $S'$, recursively selecting the (or a) highest-valued child of each max
node and all children of each product node~\citep{Poon2011}.  As when
translating an NNF for model counting, an SPN must be \emph{selective}~\citep{Peharz2014}
(the SPN equivalent of deterministic) 
for summation in the max-product semiring to give the correct MPE.\looseness=-1

\begin{corollary}
The MPE state of a selective, decomposable SPN can be found in time linear in its size.
\label{cor:mpe}
\end{corollary}

A sum node in an SPN can be viewed as the result of summing out an implicit
hidden variable $Y_v$, whose values $\Y_v = \{y_c\}_{c \in \text{Ch}(v)}$
correspond to $\text{Ch}(v)$, the children of $v$~\citep{Poon2011}. It is
often of interest to find the MPE state of both the hidden and observed
variables. This can be done in linear time 
and requires only that the SPN be decomposable, 
because making each $Y_v$ explicit
by multiplying each child $c$ of $v$ by the indicator $[Y_v = y_c]$ 
makes the resulting SPN $S(\Xb, \Yb)$ selective.

\section{Integration and optimization}
\label{sec:continuous-sup}

\textbf{Integration (continued).}
For non-decomposable SPFs, \alg{Decompose} must be altered to select
only a finite number of values and then use the trapezoidal rule for approximate integration.
Values can be chosen using grid search and if $S$ is Lipschitz 
continuous the grid spacing can be set 
such that the error incurred by the approximation is bounded
by a pre-specified amount. This can significantly reduce the number
of values explored in \alg{SumSPF} if combined with approximate
decomposability (Section~\ref{sec:learning}), since \alg{SumSPF} 
can treat
some non-decomposable product nodes as decomposable, 
avoiding the expensive call to \alg{Decompose} 
while incurring only a bounded integration error.


\section{Relational inference}
\label{sec:relational-sup}

Let $\X$ be a finite set of constants and let $R^k = \X^k$ be the complete 
relation\footnote{A relation is a set of tuples; see~\citet{Abiteboul1995} for details on relational databases.}
of arity $k$ on $\X$, i.e., the set of all tuples in $\X^k$, 
where $\X^k$ is the Cartesian product of $\X$ with itself $k-1$ times.
The universe of relations with arity up to $m$ is 
$\mb{U}_m = \JI \cup \bigcup_{i = 1}^m 2^{R^i}$, 
where $2^{R^k}$ is the power set of $R^k$ and $\JI$ is the (identity) relation containing the empty tuple. 
Since union distributes over join, and both are associative and commutative,
$\mc{R} = (\mb{U}_m, \cup, \bowtie, \varnothing, \JI)$ 
is a semiring over relations, where $\bowtie$ is natural join and $\varnothing$ is the empty set
(recall that $R = R \bowtie \JI$ for any relation $R$).
Given an extensional database $\mb{R} = \{ R_i \}$ 
containing relations $R_i$ of arity up to $m$,
an SPF on $\mc{R}$, referred to as a union-join network (UJN), 
is a query on $\mb{R}$.
In a UJN $Q$, 
each $R_i \in \mb{R}$ is composed of a union of joins of unary
tuples, such that $R_i = \bigcup_{\tuple{c_1, \dots, c_r} \in R_i} \bigbowtie_{j = 1}^r c_j$,
where the leaves of $Q$ are the unary tuples $c_j$.
The $R_i$ are then combined with unions and joins to form the full UJN (query).
A UJN $Q(\Xb)$ over query variables $\Xb = (X_1, \dots, X_n)$
defines an intensional output relation $Q_{\text{ans}} = \bigcup_{\X^n} Q(\Xb)$.
Clearly, computing $Q_{\text{ans}}$ corresponds to summation in $\mc{R}$. 
Let $n_J^Q$ denote the maximum number of variables involved in a particular join over all joins in $Q$. 
The corollary below follows immediately, since
a decomposable join is a Cartesian product.

\begin{corollary}
$Q_{\text{ans}}$ of a decomposable UJN $Q$ can be computed
in time linear in the size of $Q$ if $n_J^Q$ is bounded.
\end{corollary}
\noindent Note that $n_J^Q$ can be smaller than the treewidth of $Q$,
since 
$Q$ 
composes the final output relation from many small 
relations (starting with unary tuples) via a relational form of determinism.
Since the size of a UJN depends both on the input relations and the query,
this is a statement about the combined complexity of queries defined by UJNs.

Regarding expressivity, selection in a UJN can be implemented as a join with the relation 
$P_\sigma \in \mb{U}_m$, 
which contains all tuples that satisfy the selection predicate $\sigma$.
Projection is not immediately supported by $\mc{R}$, but since union distributes over projection,
it is straightforward to extend the results of Theorem~\ref{thm:spt} to allow UJNs to contain projection nodes. 
UJNs with projection correspond to non-recursive Datalog queries 
(i.e., unions of conjunctive queries), 
for which decomposable UJNs are a tractable sub-class.
Thus, \alg{SumSPF} defines a recursive algorithm for evaluating non-recursive Datalog queries 
and the Generic-Join algorithm~\citep{Ngo2014} --
a recent join algorithm that achieves worst-case optimal performance
by 
recursing on individual tuples 
-- is an instance
of~\alg{Decompose}.


Another consequence of the sum-product theorem is a much simpler proof
of the tractability of tractable Markov logic~\citep{Domingos2012}.

\section{Relational probabilistic models}
\label{sec:tml-sup}

A tractable probabilistic knowledge base (TPKB)~\citep{Niepert2015,Webb2013,Domingos2012} 
is a set of class and object declarations such 
that the classes form a forest and the objects form a tree of subparts
when given the leaf class of each object.
A class declaration for a class $\mtt{C}$ specifies the subparts $\mtt{Parts(C) = \{P_i\} }$, 
(weighted) subclasses $\mtt{Subs(C) = \{S_i\}}$, attributes $\mtt{Atts(C) = \{A_i\}}$, 
and (weighted) relations $\mtt{Rels(C) = \{R_i\}}$.
The subparts of $\mtt{C}$ are parts that every object of class $\mtt{C}$ must have
and are specified by a name $\mtt{P_i}$, a class $\mtt{C_i}$, and a number $n_i$ of unique copies. 
A class $\mtt{C}$ with subclasses $\mtt{S_1, \dots, S_j}$ must belong to exactly
one of these subclasses, where the weights $w_i$ specify the distribution over subclasses.
Every attribute has a domain $\mb{D_i}$ and a weight function $\mb{u_i} : \mb{D_i}\rightarrow\R$.
Each relation $\mtt{R_i(\dots)}$ has the form $\mtt{R_i(P_a, \dots, P_z)}$ where each of $\mtt{P_a, \dots, P_z}$ is a
part of $\mtt{C}$. Relations specify what relationships may hold among the subparts. 
A weight $v_i$ on $\mtt{R_i}$ defines the probability that the relation is true. A relation can also
apply to the object as a whole, instead of to its parts.
Object declarations introduce evidence by specifying an object's subclass memberships, 
attribute values, and relations as well as specifying the names and path of the object from the top
object in the part decomposition.

A TPKB $\mc{K}$ is a DAG of objects and their properties (classes, attributes, and relations), and a possible world
$\mb{W}$ is a subtree of the DAG with values for the attributes and relations.
The literals are the class membership, attribute, and relation atoms and their negations and thus
specify the subclasses of each object, the truth value of each relation, and the value of each attribute.
%
A single (root) top object $\mtt{(O_0,C_0)}$ has all other objects as descendants. 
No other objects are of top class $\mtt{C_0}$.
The unnormalized distribution $\phi$ over possible subworlds $\mb{W}$ is defined recursively as 
$\phi(\mtt{O,C,}\mb{W}) = 0$ if $\neg \mtt{Is(O,C)} \in \mb{W}$ or if a relation $\mtt{R}$ of $\mtt{C}$ is hard and 
$\neg \mtt{R(O,\dots)} \in \mb{W}$, and otherwise as
%
\begin{align}
\phi(\mtt{O},\mtt{C},\mb{W}) =&  
       \left( \sum_{\mtt{S_i \in Subs(C)}} e^{w_i} \phi(\mtt{O, S_i,} \mb{W}) \right)  \times \nonumber \\
    & \left( \prod_{\mtt{P_i \in Parts(C)}} \phi(\mtt{O.P_i, C_i,} \mb{W}) \right)  \times \nonumber \\
    & \left( \prod_{\mtt{A_i \in Atts(C)}} \alpha(\mtt{O, A_i,} \mb{W}) \right)  \times \nonumber \\
    & \left( \prod_{\mtt{R_i \in Rels(C)}} \rho(\mtt{O, R_i,} \mb{W}) \right),  \label{eqn:tml}
\end{align}
where 
$\alpha(\mtt{O,A_i,}\mb{W}) = e^{\mb{u_i}(D)}$ if $\mtt{A_i(O,D)} \in \mb{W}$ and
$\rho(\mtt{O,R_i,}\mb{W}) = e^{v_i}[\mtt{R_i}(\dots)] + [\neg \mtt{R_i}(\dots)]$.
Note that $\mtt{Parts(C)}$ contains all subparts of $\mtt{C}$, including all duplicated parts.
%
The probability of a possible world $\mb{W}$ is $\frac{1}{Z_{\mc{K}}} \phi(\mtt{O_0, C_0}, \mb{W})$
where the sub-partition function for $(\mtt{O,C})$ is 
$Z_{\mc{K}_{\mtt{O,C}}} = \sum_{\mb{W} \in \mc{W}} \phi(\mtt{O, C,} \mb{W})$ and 
$Z_\mc{K} = Z_{\mc{K}_{\mtt{O_0,C_0}}}$.

By construction, $\phi(\mtt{O_0, C_0,} \mb{W})$ defines an SPN over the literals.
With the sum-product theorem in hand, it is possible to greatly simplify the two-page proof of tractability given
in~\citet{Niepert2015}, as we show here. 
To prove that computing $Z_{\mc{K}}$ is tractable it suffices to
show that (\ref{eqn:tml}) is decomposable or can be decomposed efficiently.
We first note that
each of the four factors in (\ref{eqn:tml}) is decomposable, since
the first is a sum,
the second is a product over the subparts of $O$ and 
therefore its subfunctions have disjoint scopes,
and the third and fourth are products over the attributes and relations, respectively, and are decomposable
because none of the $\alpha$ or $\rho$ share variables. It only remains to show that the factors
can be decomposed
with respect to each other without increasing the size of the SPN.
Let $n_O, n_C, n_r$ denote the number of object declarations, class declarations,
and relation rules, respectively.
The SPN corresponding to $\phi(\mtt{O_0, C_0,} \mb{W})$
has size $O(n_O(n_C + n_r))$, since for each object $(\mtt{O,C})$ 
there are a constant number of edges for each of its
relations and subclasses.
Similarly, $\mc{K}$ has size $|\mc{K}| = O(n_O(n_C + n_r)$.
We can thus prove the following result.

\vspace{0.5em}
\begin{corollary}
The partition function of a TPKB can be computed in time linear in its size.
\end{corollary}
\vspace{0.5em}

\section{Experimental details}
\label{sec:exper-sup}

All experiments were run on the same MacBook Pro with 2.2 GHz Intel Core i7 processor 
with 16 GB of RAM. Each optimization was limited to a single thread.

\begin{figure}[t]
\includegraphics[width=1.0\columnwidth]{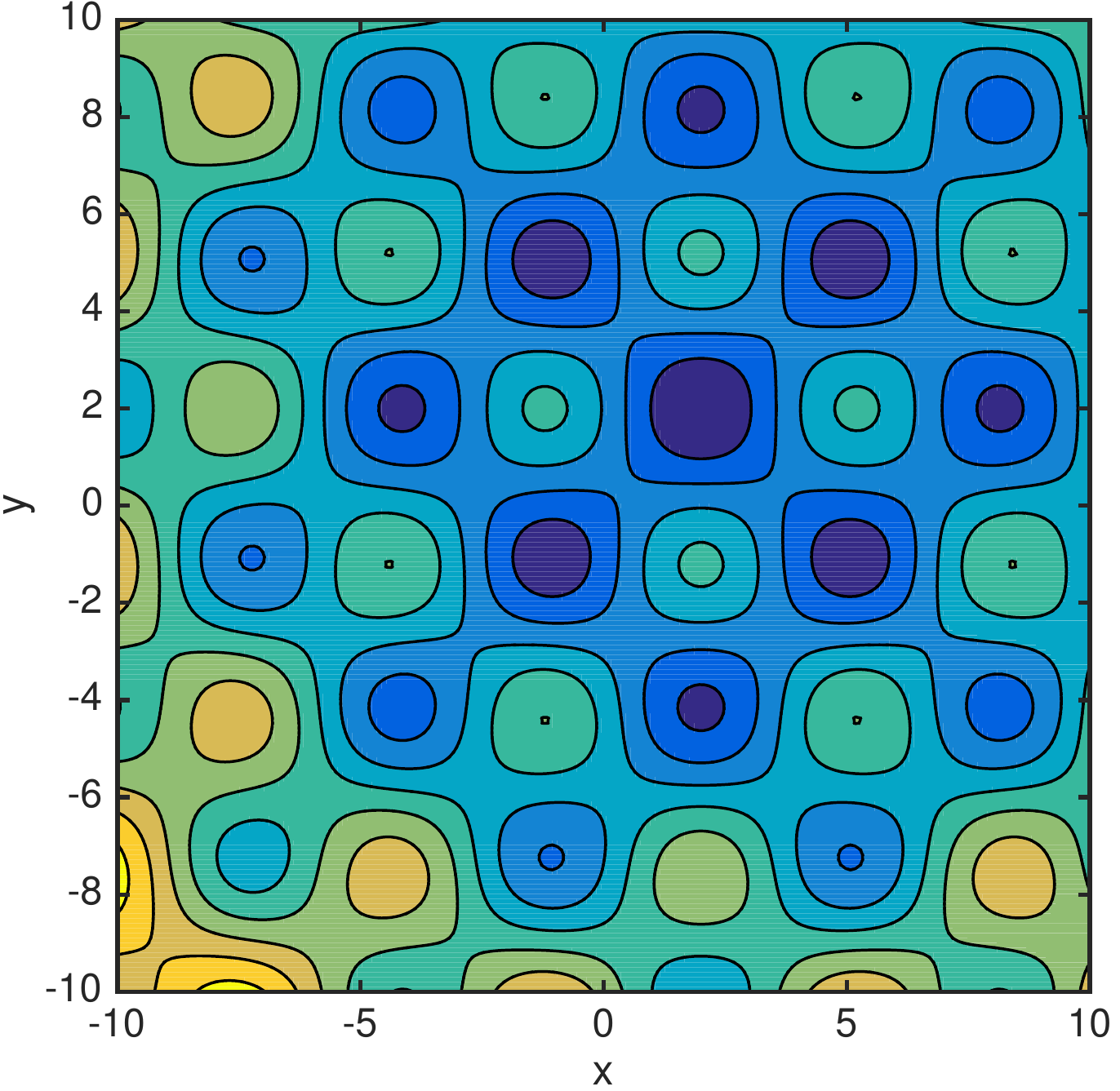}
\caption{\small{Contour plot of the 2-D nonconvex Rastrigin function.}
}
\label{fig:testfunc-sup}
\end{figure} 

The non-separable variant of the Rastrigin function~\citep{Torn1989} used on pairs of 
variables is defined as 
\begin{align}
f^R_{x_i, x_j}(Y_i, Y_j) =~& c_0  [(Y_i - x_i)^2 - (Y_j - x_j)^2] + \nonumber \\
& c_1 - c_1  \cos(Y_i - x_i) \cos(Y_j - x_j), \nonumber
\end{align}
which has a global minimum at $\y^* = (x_i, x_j)$ with value $f^R_{\x}(\y^*) = 0$.
The constants $c_0$ and $c_1$ control the shape of the quadratic basin and the amplitude
of the sinusoids, respectively. For our tests, we used $c_0 = 0.1$ and $c_1 = 20$.
Figure~\ref{fig:testfunc-sup} shows a contour plot of $f^R$.

Omitting the parameters $\x$ from $f^R_{\x}$ for simplicity, 
the full test function for $n = 4m$ variables is defined as 
$$F_{\x}(\Yb) = \sum_{i=0}^m f^R(Y_{4i}, Y_{4i + k}) + f^R(Y_{4i + 3}, Y_{4i + 3 - k}),$$ 
where $k=1$ with probability $0.5$ and $k=2$ otherwise.
This creates a function that is non-decomposable between each pair of variables 
$(Y_{4i}, Y_{4i + k})$ and $(Y_{4i + 3}, Y_{4i + 3 - k})$.
For the simplest case with $n=4$ variables, if $k=1$ then pairs $(Y_0, Y_1)$ and $(Y_2, Y_3)$
are non-decomposable. Alternatively, if $k=2$ then pairs $(Y_0, Y_2)$ and $(Y_1, Y_3)$ are
non-decomposable.
The global minimum $(x_i, x_j)$ for each function $f^R_{x_i, x_j}(Y_i, Y_j)$ was sampled
uniformly over an interval of length $2$ from the line $Y_i = Y_j$ with zero-mean additive Gaussian noise $(\sigma = 0.1)$.
Thus, each instance of $F_{\x}$ is highly nonconvex and is decomposable with respect to certain variables and not
with respect to others. For a set of instances of $F_{\x}$, there is structure in
the decomposability between variables, but different instances have different
decomposability structure, so \alg{LearnSPF} must first group those function instances
that have similar decomposability structure and then identify that structure
in order to learn a min-sum function that is applicable to any instance in the training data.

\section{Proofs}
\label{sec:proofs-sup}

\setcounter{thm}{0}
\setcounter{lemma}{0}
\setcounter{corollary}{0}

\begin{corollary}
Every SPF with bounded treewidth can be summed in time linear in the cardinality of its scope.
\label{cor:JTisSPFSUP}
\end{corollary}

\vspace{-1em}
\begin{proof}
Let $\Xb = (X_1, \dots, X_n)$,
let $(\sr,\oplus,\otimes,0,1)$ be a commutative semiring, and
let $F(\Xb)$ be an SPF with bounded treewidth $tw(F) = a$ for $0 < a < \infty$.
Let $S(\Xb)$ be a tree-like SPF that is compatible with $F$ 
and has junction tree $\mc{T} = (T,Q)$ with treewidth $tw(\mc{T}) = a$.
The size of the largest cluster in $\mc{T}$ is $\alpha = a + 1$.
Let $m = {|Q|} \leq n$ and $d = |\X_v|$ for all $X_v \in \Xb$.
Further, other than the root $s \in S$, there is a one-to-one 
correspondence between separator instantiations $\mb{s}_{ij} \in \X_{\mb{S}_{ij}}$
and sum nodes $s_{ij} \in S$, and between cluster instantiations 
$\mb{c}_j \in \X_{\Cb_j}$ and product nodes $c_j \in S$. 
Now, the number of edges in $S$ can be obtained by counting the edges that correspond to each
edge in $T$ and summing over all edges in $T$, as follows. By construction, 
each edge $(j,k) \in T$ corresponds
to the product nodes $\{c_k\}$; their children, which are the leaf nodes (indicators and constants)
and the sum nodes $\{s_{jk}\}$; and the children of $\{s_{jk}\}$, which are the product nodes $\{c_j\}$.
By definition, the $\{c_j\}$ have only a single parent, so there are $|\X_{\Cb_j}| \leq d^{\alpha}$ edges between 
$\{s_{jk}\}$ and $\{c_j\}$. Further, each $c_k$ has only $|\Cb_k|+1$ leaf node children
and $|\text{Ch}(k)|$ sum node children, so there are 
$|\X_{\Cb_k}| (|\Cb_k|+1) (|\text{Ch}(k)|) \leq d^{\alpha} (\alpha+1) (|\text{Ch}(k)|)$ edges
between $\{c_k\}$ and $\{s_{jk}\}$. 
In addition, there are also $\X_{\Cb_r} = d^\alpha$ edges between the root $s \in S$ and the
product nodes $c_r$.
Thus, since $T$ is a tree with $m-1$ edges, 
$\text{size}(S) \leq d^{\alpha} + \sum_{(j,k) \in T} 2 d^{\alpha} (\alpha+1) (|\text{Ch}(k)|) = O(m d^{\alpha})$, 
which is $O(n)$.
Since $S$ is decomposable and has size $O(n)$, then, from the sum-product theorem, 
$S$ can be summed in time $O(n)$. Furthermore, $S$ is compatible with $F$,
so $F$ can be summed in time $O(n)$, and the claim follows.
\end{proof}

\begin{corollary}
Not every SPF that can be summed in time linear in the cardinality of its scope has bounded treewidth.
\label{cor:SPFnotJTSUP}
\end{corollary}
\begin{proof}
By counterexample. 
Let $\Xb = (X_1, \dots, X_n)$ be a vector of variables,
$(\sr, \oplus, \otimes,0,1)$ be a commutative semiring, 
and $k = {|\mc{\X}_i|}$ for all $X_i \in \Xb$.
The SPF $F(\Xb) = \bigoplus_{j=1}^r \bigotimes_{i=1}^n \psi_{ji}(X_i)$
can be summed in time linear in $n$ because $F$ is decomposable
and has size $r(n+1)$.
At the same time, $F(\Xb)$ has treewidth $n - 1$ (i.e., unbounded)
because there are no pairwise-disjoint subsets $\mb{A}, \mb{B}, \Cb \subseteq \Xb$ with
domains $\X_A,\X_B,\X_C$ such that
$\mb{A}$ and $\mb{B}$ are conditionally independent in $F$ given $\Cb$, and thus the smallest
junction tree compatible with $F(\Xb)$ is a complete clique over $\Xb$. This can be seen as follows. 
Without loss of generality, let $\mb{A} \cup \mb{B}$ be the first $m$ variables in $\Xb$, $(X_1, \dots, X_m)$.
For any $\mb{c} \in \X_{C}$, 
$F(\mb{A},\mb{B}, \mb{c}) \propto \bigoplus_{j=1}^r \bigotimes_{i : X_i \in \mb{A} \cup \mb{B}} \psi_{ji}(X_i) 
= ( \psi_{11}(X_1) \otimes \cdots \otimes \psi_{1m}(X_m) ) \oplus \dots \oplus 
( \psi_{r1}(X_1) \otimes \cdots \otimes \psi_{rm}(X_m) )$. For $F(\mb{A}, \mb{B},\mb{c})$ to factor,
the terms in the right-hand side must have common factors; however, in general,
each $\psi_{ji}$ is different, so there are no such factors. Thus, 
$F(\mb{A}, \mb{B}, \mb{c}) \neq F(\mb{A},\mb{c}) \otimes F(\mb{B}, \mb{c})$
for all $\mb{c} \in \X_C$, and there are no conditional independencies in $F$.
\end{proof}

\setcounter{corollary}{7}
\begin{corollary}
The model count of a deterministic, decomposable NNF can be computed in time linear in its size.
\label{cor:sharpsatSUP}
\end{corollary}

\begin{proof}
Let $F(\Xb)$ be a deterministic, decomposable NNF and $F'(\Xb)$ 
be $F$ translated to the sum-product semiring.
Clearly, $F$ and $F'$ have equal size and $F'$ is deterministic and decomposable.
Thus, from the sum-product theorem, $\sum_{\X} F'(\Xb)$ 
takes time linear in the size of $F$.
Let $v$ be a node in $F$ and $v'$ its corresponding node in $F'$.
It remains to show that $\sum_{\X} F_v'(\Xb) = \shSAT( F_v(\Xb) )$ for all $v,v'$, 
which we do by induction.
The base case with $v$ a leaf node holds trivially. 
For the induction step, assume that $\sum_{\X_i} F_i'(\Xb_i) = \shSAT( F_i(\Xb_i) )$
for each child $c_i \in \text{Ch}(v)$ (resp. $c_i' \in \text{Ch}(v')$).
If $v$ is an AND node then $v'$ is a multiplication node and 
$\shSAT(F_v(\Xb)) = \shSAT( \bigwedge_{c_i} F_i(\x_i) ) = 
\prod_{c_i} \shSAT( F_i(\x_i) ) = 
\sum_{\X} F_v'(\Xb)$, because $v$ and $v'$ are decomposable.
If $v$ is an OR node then $v'$ is an addition node and
$\shSAT( F_v(\Xb) ) = \shSAT( \bigvee_{c_i} F_i(\x_i) ) = 
\sum_{c_i} \shSAT( F_i(\x_i) ) = 
\sum_{\X} F_v'(\Xb)$, because $v$ is deterministic, 
so its children are logically disjoint.
\end{proof}

\setcounter{corollary}{9}
\begin{corollary}
The MPE state of a selective, decomposable SPN can be found in time linear in its size.
\label{cor:mpeSUP}
\end{corollary}

\begin{proof}
Let $S(\Xb)$ be a selective, decomposable SPN and $S'(\Xb)$ its max-product version, 
which has the same size and is also selective and decomposable.
For clarity, we assume no evidence since it is trivial to incorporate.
From the sum-product theorem, $\max_{\X} S'(\Xb)$ takes time linear in the size of $S$.
Let $v$ be a node in $S$ and $v'$ its corresponding node in $S'$.
It remains to show that $\max_{\X} S_v'(\Xb) = \max_{\X} S(\Xb)$ for all $v, v'$,
which we do by induction on $v$.
The base case with $v$ a leaf holds trivially because $v$ and $v'$ are identical.
For the induction step, assume that $\max_{\X_i} S_i'(\Xb_i) = \max_{\X_i} S_i(\Xb_i)$ for
each child $c_i \in \text{Ch}(v)$ (resp. $c_i' \in \text{Ch}(v')$). 
If $v$ is a product node then so is $v'$ and $\max_{\X} S_v'(\Xb) = \max_{\X} S_v(\Xb)$.
If $v$ is a sum node then $v'$ is a max node and
$\max_{\X} S(\Xb) = \max_{\x \in \X} \sum_{c_i} S_i(\x_i) =
\max_{\x \in \X} \{ \max_{c_i} S_i(\x_i) \} =
\max_{c_i} \{ \max_{\x_i \in \X_i} S_i(\x_i) \} = \max_{\X} S'(\Xb)$, 
where the second equality occurs because $v$ is selective.
%
After summing $S'$, the MPE state is recovered by a downward pass in $S'$, 
which takes linear time.
\end{proof}

\setcounter{corollary}{10}
\begin{corollary}
The partition function of TPKB $\mc{K}$ can be computed in time linear in its size.
\end{corollary}

\begin{proof}
The only sources of non-decomposability in (\ref{eqn:tml}) are
if an object $(\mtt{O,C})$ and one of its subclasses $(\mtt{O,S_j})$ both contain
(i) the same relation $\mtt{R_i(O,\dots)}$ or 
(ii) the same attribute $\mtt{A_i(O,D)}$.
Note that they cannot contain the same part since two classes such that one is a
descendant of the other in the class hierarchy never have a
part with the same name.
In each of the above cases, the shared relation (or attribute) can be pushed into
each subclass $(O,S_j)$ by distributing $\rho(\mtt{O,R_i,}\mb{W})$ (or $\alpha(\mtt{A_i,C_i,} \mb{W})$)
over the subclass sum and into each subclass (this can be repeated for multiple levels
of the class hierarchy). This makes the product over relations (attributes) in $\phi(\mtt{O,S_j,} \mb{W})$ 
non-decomposable, but does not affect the decomposability of any other objects. 
Now, $\phi(\mtt{O,S_j,} \mb{W})$ can be decomposed, as follows.
For (i), the product over relations in $\phi(\mtt{O,S_j,} \mb{W})$ 
now contains the non-decomposable factor $F(\mtt{R_i}) = \rho(\mtt{O,R_i,} \mb{W}) \cdot \rho'(\mtt{O,R_i,} \mb{W})$,
where $\rho'$ was pushed down from $\mtt{(O,C)}$. 
However, $F(\mtt{R_i}) = (e^{w_i}[\mtt{R_i}] + [\neg \mtt{R_i}]) \cdot (e^{w_i'}[\mtt{R_i}] + [\neg \mtt{R_i}])
= e^{w_i + w_i'}[\mtt{R_i}] + [\neg \mtt{R_i}]$ since $[a]^2 = [a]$ and $[a][\neg a] = 0$ for a literal $a$.
Thus, $F(\mtt{R_i})$ is simply $\rho(\mtt{O,R_i,}\mb{W})$ with weight $w_i + w_i'$ for $\mtt{R_i}$, which
results in the same decomposable SPN structure with a different weight.
For (ii), 
let the weight function for attribute $\mtt{A_i}$ of class $\mtt{C}$ with domain $\mb{D}_i$ 
be $\mb{u}_i$ and the weight function from the attribute that was pushed down be $\mb{u}_i'$.
To render this decomposable, simply replace $\mb{u}_i$ with $\mb{u}_i \odot \mb{u}_i'$,
the element-wise product of the two weight functions. Again, decomposability is achieved simply 
by updating the weight functions.
Decomposing (i) and (ii) each adds only a linear number of edges to the original 
non-decomposable SPN, so the size of the corresponding decomposable SPN is $|\mc{K}|$.
Thus, from the sum-product theorem, computing the partition function of TPKB $\mc{K}$ takes
time linear in $|\mc{K}|$.
\end{proof}

\end{appendices}

\end{document}